\documentclass[journal]{IEEEtran}

\usepackage{inputenc}
\usepackage{mathtools}
\usepackage{color}
\usepackage{array}
\usepackage{verbatim}
\usepackage{float}
\usepackage{amsmath}
\usepackage{amsthm}
\usepackage{amssymb}
\usepackage{graphicx}
\usepackage{longtable}
\usepackage{multirow}
\usepackage{balance}
\usepackage{mathtools}

\usepackage[unicode=true,
bookmarks=false,
breaklinks=false,pdfborder={0 0 1},colorlinks=false]
{hyperref}
\hypersetup{
	colorlinks,bookmarksopen,bookmarksnumbered,citecolor=red,urlcolor=red}
\usepackage{cite}

\usepackage{indentfirst}
\usepackage{textcomp}
\usepackage[T1]{fontenc}
\usepackage{amsthm}
\usepackage{amsmath,amstext,amsfonts}
\usepackage{amssymb}
\usepackage{mathtools}
\usepackage[cmintegrals]{newtxmath}
\usepackage{bm}  
\usepackage{siunitx}
\usepackage{physics}

\usepackage{graphicx}
\graphicspath{%
	{Figures/}
	{Figures/TiKz/}
	{Figures/Ipe/}
	{Figures/XFig/}
	{Figures/Inkscape/}
	{Figures/Matlab/}
	{Figures/Scanned/}
	{Figures/Saved/}
}
\usepackage{tikz}
\usepackage[list=true,labelformat=simple]{subcaption}
\usepackage{cite}
\usepackage{url}
\usepackage{xcolor}
\usepackage{siunitx}

\usepackage[inline]{enumitem}

\usepackage{algorithm,algpseudocode,float}
\usepackage{setspace} 

\algnewcommand{\algorithmicand}{\textbf{ and }}
\algnewcommand{\algorithmicor}{\textbf{ or }}
\algnewcommand{\AlgAnd}{\algorithmicand}
\algnewcommand{\AlgOr}{\algorithmicor}

\usepackage[noabbrev]{cleveref}  
\Crefname{figure}{Fig.}{Figs.}

\usepackage{tabularx,booktabs}
\newcolumntype{C}{>{\centering\arraybackslash}X} 
\usepackage{lipsum}

\makeatletter
\let\oldforeign@language\foreign@language
\DeclareRobustCommand{\foreign@language}[1]{%
	\lowercase{\oldforeign@language{#1}}}

\floatstyle{ruled}
\newfloat{algorithm}{tbp}{loa}
\providecommand{\algorithmname}{Algorithm}
\floatname{algorithm}{\protect\algorithmname}

\let\oldforeign@language\foreign@language
\DeclareRobustCommand{\foreign@language}[1]{%
	\lowercase{\oldforeign@language{#1}}}

\ifCLASSINFOpdf
\else
\fi

\newtheorem{defn}{Definition}

\newtheorem{lem}{Lemma}

\newtheorem{thm}{Theorem}

\newtheorem{assum}{Assumption}
\ifCLASSINFOpdf
\else
\fi

\def\ps@IEEEtitlepagestyle{%
	\def\@oddhead{\parbox[t][\height][t]{\textwidth}{\centering \scriptsize
			Personal use of this material is permitted. Permission from the author(s) and/or copyright holder(s), must be obtained for all other uses. Please contact us and provide details if you believe this document breaches copyrights.\\
			\noindent\makebox[\linewidth]{}
		}\hfil\hbox{}}%
	\def\@evenhead{\scriptsize\thepage \hfil \leftmark\mbox{}}%
	\def\@oddfoot{\parbox[t][\height][l]{\textwidth}{
			\vspace{-20pt}{\rule{\textwidth}{0.4pt}}\\ \footnotesize			{\bf{\footnotesize\textcolor{red}{A. M. Ali, C. Shen, and H. A. Hashim, "A Linear MPC with Control Barrier Functions for Differential Drive Robots," IET Control Theory \& Applications, 2024.}}}\\
			\noindent\makebox[\linewidth]
		}\hfil\hbox{}}%
	\def\@evenfoot{\MYfooter}}

\makeatother
\begin{document}
	\bstctlcite{IEEEexample:BSTcontrol}

\title{A Linear MPC with Control Barrier Functions for Differential Drive Robots}

\author{Ali Mohamed Ali, Chao Shen, and Hashim A. Hashim
	\thanks{This work was supported in part by the National Sciences and Engineering Research Council of Canada (NSERC), under the grants RGPIN-2022-04937, RGPIN-2022-04940, DGECR-2022-00103 and DGECR-2022-00106.}
	\thanks{A. M. Ali and H. A. Hashim are with the Department of Mechanical
		and Aerospace Engineering, Carleton University, Ottawa, ON, K1S-5B6,
		Canada (e-mail: AliMohamedAli@cmail.carleton.ca and hhashim@carleton.ca). C. Shen is with the Department of Systems and Computer Engineering, Carleton
		University, Ottawa, ON, K1S-5B6, Canada (shenchao@sce.carleton.ca).}
}



\maketitle
\begin{abstract}
The need for fully autonomous mobile robots has surged over the past decade, with the imperative of ensuring safe navigation in a dynamic setting emerging as a primary challenge impeding advancements in this domain. In this paper, a Safety Critical Model Predictive Control based on Dynamic Feedback Linearization tailored to the application of differential drive robots with two wheels is proposed to generate control signals that result in obstacle-free paths. A barrier function  introduces a safety constraint to the optimization problem of the Model Predictive Control (MPC) to prevent collisions. Due to the intrinsic nonlinearities of the differential drive robots, computational complexity while implementing a Nonlinear Model Predictive Control (NMPC) arises. To facilitate the real-time implementation of the optimization problem and to accommodate the underactuated nature of the robot, a combination of Linear Model Predictive Control (LMPC) and Dynamic Feedback Linearization (DFL) is proposed. The MPC problem is formulated on a linear equivalent model of the differential drive robot rendered by the DFL controller. The analysis of the closed-loop stability and recursive feasibility of the proposed control design is discussed. Numerical experiments illustrate the robustness and effectiveness of the proposed control synthesis in avoiding obstacles with respect to the benchmark of using Euclidean distance constraints.
\end{abstract}


\section{Introduction}\label{sec1}

\subsection{Motivation}
\IEEEPARstart{S}{afety}, stability, and optimality of control systems
are fundamental problems that have tight conflicting coupling \cite{AaronIEEE}.
The increased deployment of mobile robots in industries such as manufacturing,
healthcare, and logistics, encouraged researchers to develop robust
and comprehensive control architectures over the past decade \cite{IET-Trakinghuman}.
Starting with the early work in \cite{KanayamaICRA} where a local
Lyapunov-based control design has been proposed, followed by \cite{OrioloIEEETrans} 
a Dynamic Feedback Linearization (DFL) and thereafter backstepping
tracking control approach \cite{SaturationfeedbackIEEETrans}
for the global trajectory tracking of unicycle robots. Addressing
parametric uncertainties, sliding Mode Control (SMC) has been investigated by
targeting robot stabilization and trajectory tracking \cite{IntegralSlidingFridmanIEEETrans,ferraraIETControlTheory,hwang2013trajectoryIET,IETUnicycleFintetime,IETSliding2010}.
Prescribed transient performance and guaranteed steady error have
been addressed using prescribed control technique \cite{IETprescribedperformance2015}
where time-varying reducing boundaries have been imposed for the robot
tracking position. However, the above-listed approaches overlooked
control input and state constraints, which hinders the control practicality
and feasible implementation.

\subsection{Scope and Literature}

\hspace{0.3cm} Model predictive control (MPC), also known as receding horizon control,
is an advanced control approach that was invented for industrial process
control and gained popularity because it considers
control input and state constraints \cite{garcia1989bookmpc}. MPC
has been introduced for unicycle model path following in  \cite{bookoptimization}-\cite{MPCmobileIFAC2018}
considering the nominal model. The work in \cite{SunIEEETrans2017} modified
the MPC to address unicycle robot input constraints considering small
bounded disturbances. In \cite{MPCmobileIFAC2018}, an MPC tightly
coupled with a nonlinear disturbance observer has been designed to
estimate and compensate for external disturbances. Although the above-mentioned
papers were able to utilize MPC to successfully follow a predefined
path, they do not guarantee safe path following of mobile robots in
a typical unknown working environment. A potential conflict for satisfying
control input and state constraints with other safety criteria may
arise. In other words, not all predefined paths would be safe
\cite{Aaron2019ECC}.

\hspace{0.3cm}Safe navigation with the MPC framework is typically enforced as distance
constraints described in the form of Euclidean norms where the distance
between the navigating robot and obstacles should be larger than a
safety margin, see \cite{Ali2024ACC,YOON2009741,6728261,7489011}. The distance
constraint will not confine the optimization problem until a reachable
set along the horizon intersects with the obstacles. This way, the
robot will not take action to avoid the obstacles unless they are
nearby. One way to address this challenge is by using a larger horizon,
which could significantly increase the computational cost in real-time
implementation. As such, there is a need for an invariant safe set
that could confine the robot’s movement during the optimization at
every time step independent of whether the robot is near an obstacle
or not \cite{Aaron2017IEEETrans}.
Control barrier functions (CBF) represent a safety measure of the
system that could be utilized (e.g., Euclidean distance between mobile
robot and obstacles). In this case, the controller can be synthesized
and the CBF can be designed to guarantee stable error dynamics as
well as safe navigation. \textcolor{black}{Recently, CBF was introduced
	in \cite{Aaron2019ECC}, where the control input was the solution
	to a quadratic programming problem. The main idea of the quadratic
	programming program is to impose a minimal invasive change of the
	stabilizing controller to be also safe \cite{Ali2024ACC, Aaron2017IEEETrans}.
	However, those contributions lack the prediction capabilities
	in the formulation of the MPC formulation.} Recent efforts considered
unifying CBF with the traditional MPC formulation. A unified framework
of MPC with CBF was introduced to control the Segway model \cite{Grandia2020NonlinearMP}.
In \cite{SonCDC2019}, the authors developed nonlinear MPC based on
CBF for vehicle avoidance, however, no theoretical guarantee of closed-loop stability has been provided. MPC and CBF are organized as high-level
planner and a low-level tracker in \cite{rosolia2020multi}, and sufficient
conditions which guarantee recursive constraint satisfaction for the
closed-loop system were provided.

\hspace{0.3cm}It has been well-recognized that real-time implementation
of nonlinear MPC is subject to computational complexity when compared
to Linear MPC (LMPC) schemes. Applying MPC on a linear dynamical model with quadratic cost, functions renders the optimization problem to
be a quadratic program which usually leads to a fast online solution.
To implement the nonlinear MPC in real time, at each sampling instant,
a nonlinear open-loop optimal control problem has to be solved within
strict time constraints. Violation of the time constraints could 
degrade the output performance and/or stability measures \cite{CompdelayMPC2004IFAC}
and thereby fast nonlinear MPC techniques are studied by many researchers for instance see \cite{ZanonFastsolver2015, ChaoFastSolver2017IEEETran, OHTSUKA2004Automatica}. Several robotic applications consider feedback linearization techniques. The
work in \cite{CHARLET1989143} introduced DFL for the Multi-Input
Multi-Output (MIMO) system, where the system can be modified to an
equivalence linear model consisting of a chain of integrators under
a feedback controller and a proper change of coordinates. A linear
equivalence model has been derived for unicycle robots \cite{OrioloIEEETrans}
and car-like robots \cite{Grandia2020NonlinearMP}. The linear equivalence
model unlocks the possibility of applying MPC on a linear model instead
of a nonlinear one. In \cite{KONG2023126658}, a coupled MPC with
an Input/Output Feedback Linearization (IOFL) approach illustrated
improvement in the thermal power plant economic and dynamic output
performance with fast real-time implementation. MPC cascaded with
Feedback linearization has been developed for the fully-actuated spacecraft
attitude model \cite{IET2023MPCFL}.
To the best of the authors’ knowledge, the cascaded scheme of MPC and feedback
linearization with obstacle avoidance tailored to the application of differential
drive robots has not yet been addressed.

\subsection{\textcolor{black}{Contributions and Structure}}

\hspace{0.3cm}
The proposed solution minimizes the above-identified literature gaps by formulating MPC using CBF to control the nonlinear model of two-wheeled differential robots. The safety constraint gets activated everywhere not only when the vehicle is near an obstacle as in the case of the usage of Euclidean norms as a constraint. The CBF will provide the notion of the global forward invariance of the safe set, in other words, the robot will avoid the obstacle even if it is far from it leading to a shorter prediction horizon. Unlike the majority of the existing mobile robots collision avoidance literature, our solution considers the full nonlinear underactuated model. To address the system nonlinearities and the computational complexity arising from the nonlinear MPC, we introduce a solution that integrates a cascaded scheme of DFL with MPC unlocking the benefits of linear MPC as opposed to nonlinear MPC. The contributions of this work can
be summarized as follows:
\begin{enumerate}
	\item A cascaded scheme of DFL and
	MPC is proposed to address the nonlinear MPC computational complexity
	due to the intrinsic nonlinearity of the robot.
	\item The mapping between the original nonlinear underactuated
	model and the linear equivalent model (the MPC-CBF formulation is
	designed on the linear equivalent model rendered by the DFL) is presented.
	Combining MPC-CBF and DFL into a single scheme allows us to convert
	the obstacle avoidance of the full model of Unicycle to a Quadratic
	Constraint Quadratic Programming (QCQP) Problem that can be solved efficiently by off-the-shelf solvers.
	\item Closed loop stability, recursive  feasibility  and computational complexity of the proposed scheme are analyzed, and
	numerical simulations for a standard safe navigation task of a two-wheeled differential drive robot are carried out demonstrating the effectiveness of the proposed scheme.
\end{enumerate} 
The remaining part of the paper is organized as
follows. Section \ref{sec:Problem-Formulation} presents preliminaries,
mathematical notation, and problem formulation. A brief description
of the linear equivalence model is discussed in Section \ref{sec:Sec2_Linear-Equivalence}.
Section \ref{sec:Sec5_Proposed-Scheme} demonstrates the proposed
control scheme. Section \ref{sec:Sec6_Results} illustrates the effectiveness
of the proposed scheme through numerical simulations. Finally, Section
\ref{sec:Sec7_Conclusion} concludes the work.

\section{Problem Formulation}\label{sec:Problem-Formulation}

\subsection{Preliminaries}\label{sec:Preliminaries}

\hspace{0.3cm} In this paper, $\mathbb{R}$ denotes the set of real numbers, $n$ describes the degree of freedom,
and $\mathcal{L}$ defines the Lie derivative operator. For a given
vector field $f(x)$ such that $f:\mathbb{R}^{n}\rightarrow\mathbb{R}^{n}$
and a scalar function $\lambda:\mathbb{R}^{n}\rightarrow\mathbb{R}$,
the Lie derivative of $\lambda$ with respect to $f$ can be written
as $\mathcal{L}_{f}\lambda=\frac{\partial\lambda}{\partial x}\cdotp f(x).$
Consider the following single input single output nonlinear affine
in control system:

\begin{equation}
	\begin{cases}
		\dot{x} & =f(x)+g(x)u\\
		y & =h(x)
	\end{cases}\label{eq:pr1}
\end{equation}
where $x\in\mathbb{R}^{n}$ describes the system states, $u\in\mathbb{R}$
defines the system control input, $y\in\mathbb{R}$ denotes the system
output, $f:\mathbb{R}^{n}\rightarrow\mathbb{R}^{n}$, $g:\mathbb{R}^{n}\rightarrow\mathbb{R}^{n}$,
and $h:\mathbb{R}^{n}\rightarrow\mathbb{R}$. The relative degree
$r$ of such system can be defined at point $x_{0}$ if $\mathcal{L}_{g}\mathcal{L}_{f}^{\rho}h(x)=0$
for all $x$ in the neighborhood of $x_{0}$, $\rho<r-1$, and $\mathcal{L}_{g}\mathcal{L}_{f}^{r-1}h(x_{0})\neq0.$
The Multi Input Multi Output (MIMO) square affine in control system
is expressed as follows:

\begin{equation}
	\begin{cases}
		\dot{x} & =f(x)+g_{1}(x)u_{1}+\cdots+g_{m}(x)u_{m}\\
		y_{1} & =h_{1}(x)\\
		\vdots & =\vdots\\
		y_{m} & =h_{m}(x)
	\end{cases}\label{eq:pr2}
\end{equation}

\begin{lem}
	\label{lem:degree}\cite{isidori1985nonlinear} The relative degree
	of \eqref{eq:pr2} at $x_{0}$ is describes as $r=[r_{1},\ldots,r_{m}]^{\top}\in\mathbb{R}^{m}$
	such that $r$ exists if the following holds:
\end{lem}
\begin{itemize}
	\item $\mathcal{L}_{g_{j}}\mathcal{L}_{f}^{\rho}h_{i}(x)=0$ at the neighborhood
	of $x_{0}$ for all $1\leq j\leq m$, $\rho<r_{i}-1$, and $1\leq i\leq m$.
	\item The decoupling system input matrix  $A(x)\in\mathbb{R}^{m\times m}$ defined as 
\end{itemize}
\begin{equation}
	A(x)=\left(\begin{array}{ccc}
		\mathcal{L}_{g_{1}}\mathcal{L}_{f}^{r_{1}-1}h_{1}(x) & \cdots & \mathcal{L}_{g_{m}}\mathcal{L}_{f}^{r_{1}-1}h_{1}(x)\\
		\mathcal{L}_{g_{1}}\mathcal{L}_{f}^{r_{2}-1}h_{2}(x) & \cdots & \mathcal{L}_{g_{m}}\mathcal{L}_{f}^{r_{1}-1}h_{2}(x)\\
		\vdots & \cdots & \vdots\\
		\mathcal{L}_{g_{1}}\mathcal{L}_{f}^{r_{m}-1}h_{m}(x) & \cdots & \mathcal{L}_{g_{m}}\mathcal{L}_{f}^{r_{m}-1}h_{m}(x)
	\end{array}\right)\label{eq:decouplingmatrix-1}
\end{equation}

is nonsingular at $x=x_{0}$.
\begin{lem}
	\label{lem:Full}\cite{isidori1985nonlinear} The input-to-state feedback
	linearization of the system dynamics in \eqref{eq:pr2} is solvable
	at $x_{0}$ using the control input $u=A^{-1}(x)\left(\left[\begin{array}{c}
		\mathcal{L}_{f}^{r_{1}}h_{1}(x)\\
		\vdots\\
		\mathcal{L}_{f}^{r_{m}}h_{m}(x)
	\end{array}\right]+\left[\begin{array}{c}
		v_{1}\\
		\vdots\\
		v_{m}
	\end{array}\right]\right)$,where $v$ is an external reference input to be defined, if $\sum_{r=1}^{m}r=n$
	and the decoupling input matrix $A(x)$ in \eqref{eq:decouplingmatrix-1}
	is full rank.
\end{lem}
\hspace{-0.5cm} Throughout this paper  $\mathcal{H}(x)$
represents control barrier function with $\mathcal{H}(x):D\subset\mathbb{R}^{n}\rightarrow\mathbb{R}$
describing a safety metric. The class of $\mathcal{K_{\infty}}$ extended
function is donated by $\aleph$ such that it is a continuous function
with the mapping $[0,\infty)\rightarrow[0,\infty)$ and $lim_{r\rightarrow\infty}\aleph(r)=\infty$.
\begin{defn}
	Let $C=\{x\in D\subset\mathbb{R}^{n}:\mathcal{H}(x)\geq0\}$ be the
	Safe set where $Int(C)=\{x\in D\subset\mathbb{R}^{n}:\mathcal{H}(x)>0\}$
	describes the interior of the set and $\partial C=\{x\in D\subset\mathbb{R}^{n}:\mathcal{H}(x)=0\}$
	refers to the set boundary.
\end{defn}

\subsection{Model Dynamics}

\hspace{0.3cm} The unicycle dynamic model represents the class of differential wheeled
robots with two wheels that have been used extensively in many applications
for their low cost and simplicity. Differential robots with two wheels
are a class of mobile robots whose movement is based on two separately
driven wheels placed on either side of the robot body and an optional
but recommended a caster wheel to prevent the vehicle from tilting.
Assuming that the two wheels can only perform rolling, the kinematic
model can be described by the following set of nonlinear differential
equations \cite{de2002control}.

\begin{alignat}{1}
	\dot{x} & =\frac{r}{2}(\omega_{r}+\omega_{l})\cos(\theta).\label{eq:001}\\
	\dot{y} & =\frac{r}{2}(\omega_{r}+\omega_{l})\sin(\theta).\label{eq:002}\\
	\dot{\theta} & =\frac{r}{d}(\omega_{r}-\omega_{l}).\label{eq:003}
\end{alignat}
where (visit Fig. \ref{fig:figure1-1}.(b)): $r\in\mathbb{R}$ is the
radius of the wheels, $d\in\mathbb{R}$ is the wheel axis length,
$\{\omega_{r},\omega_{l}\}\in\mathbb{R}$ are the right and the left angular
velocities respectively, $x$, $y$ and $\theta$ are the center of
mass of the robot and its orientation respectively. Assuming rolling
without slipping for both wheels, the nonholonomic constraint can
be considered as follows \cite{choset2005principles}:

\begin{equation}
	\dot{x}\sin\theta=\dot{y}\cos\theta\label{eq:004}
\end{equation}
The nonholonomic constraint in \eqref{eq:004} introduces kinematic
constraints such that the robot cannot reach by suitable maneuvers
any desired value of $x$, $y$, and $\theta$. This is intuitive since
the robot cannot move directly to the left or right without rotating.
For more information regarding the nonholonomic constraints of the
model (visit \cite{de2002control}). 
\subsubsection{Unicycle Dynamic Model}

Recall the differential wheeled robots with two wheels dynamic model in \eqref{eq:001}-\eqref{eq:003}, it can be converted to an equivalent unicycle dynamic model by using
the input transformation $T$ as follows:

\begin{equation}
	\left[\begin{array}{c}
		u_{1}\\
		u_{2}
	\end{array}\right]=\underbrace{\left[\begin{array}{cc}
			\frac{r}{2} & \frac{r}{2}\\
			\frac{r}{d} & -\frac{r}{d}
		\end{array}\right]}_{T}\left[\begin{array}{c}
		\omega_{r}\\
		\omega_{l}
	\end{array}\right].\label{eq:005}
\end{equation}
where $u_{1}\in\mathbb{R}$ is the linear velocity of the robot and
$u_{2}=\dot{\theta}$ is the angular velocity of the robot. The position
and orientation of the Unicycle can be defined relative to the body
frame attached to its center of gravity and the global inertial frame
as shown in Fig. \ref{fig:figure1-1}. (a). The unicycle equivalent dynamic
model can be written as follows:

\begin{equation}
	\dot{x}=\left[\begin{array}{c}
		u_{1}\cos(x_{3})\\
		u_{1}\sin(x_{3})\\
		u_{2}
	\end{array}\right],\hspace{1em}\begin{array}{c}
		y_{1}=x_{1}\\
		y_{2}=x_{2}
	\end{array}\label{eq:nonlinearmodel-1}
\end{equation}
such that $(x_{1},x_{2},x_{3})$ denotes the unicycle robot generalized
coordinates, $x_{1}$ and $x_{2}$ defined the robot's position (x-y
coordinates), and $x_{3}$ refers to angle between the robot and the
x axis reference-frame. The system has three states ($n=3$) with
the configuration space $\mathcal{Q}\in\mathbb{R}^{2}\times SO(1)$
such that $SO(1)$ refers to the Special Orthogonal Group of order
1 (for more information see \cite{hashim2019special}).

\begin{figure}[!htbh]
	\includegraphics[width=8cm,height=6cm,keepaspectratio]{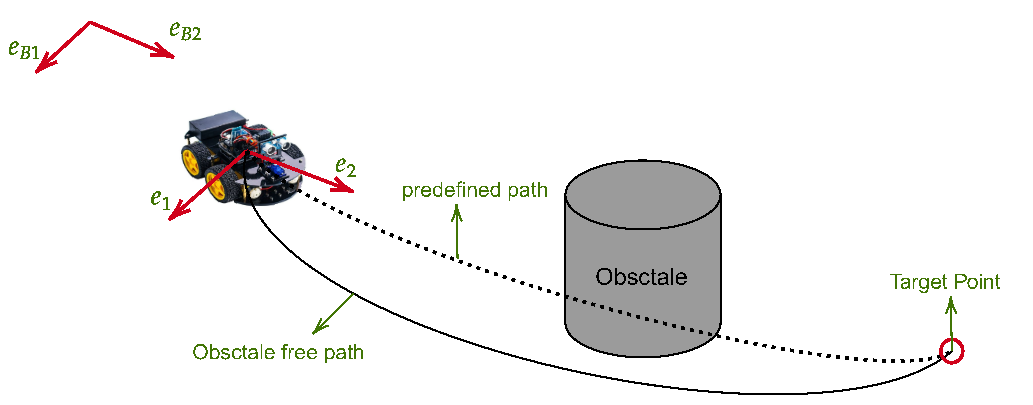}
	
	\centering (a)
	
	\includegraphics[width=10cm,height=8cm,keepaspectratio]{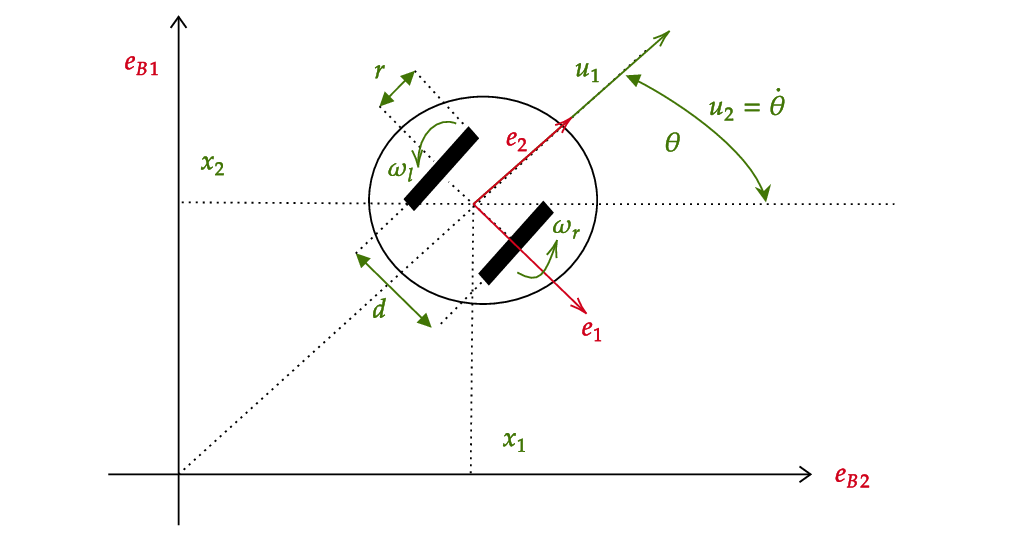}
	
	\centering (b)
	
	\caption{(a) Differential drive robots safe navigation task where $\{e_{1},e_{2}\}$ is the
		body fixed frame and $\{e_{B1},e_{B2}\}$ is the global frame. (b)
		differential drive robots model coordinates.}
	\label{fig:figure1-1}
\end{figure}
\hspace{-0.5cm} In view of \cite{minh2022safetycritical}, the authors proposed the
nonlinear MPC with CBF safety enforcement described as follows:
\begin{equation}
	\text{minimize}_{x_{k}(.),u_{k}(.)}\sum_{k=0}^{N-1}l_{k}(x_{k},u_{k})+V_{N}(x_{N})\label{eq:pr3}
\end{equation}
subject to
\begin{gather}
	x_{k+1}=f(x_{k},u_{k})\hspace{1cm}k=0,1,\ldots,N-1\label{eq:pr4}\\
	x_{min}\leq x_{k}\leq x_{max}\hspace{1cm}k=0,1,\ldots,N-1\label{eq:pr5}\\
	u_{min}\leq u_{k}\leq u_{max}\hspace{1cm}k=0,1,\ldots,N-1\label{eq:pr6}\\
	\triangle\mathcal{H}(x_{k+1},u_{k+1})\geq\gamma\mathcal{H}(x_{k})\hspace{1cm}k=0,1,\ldots,N-1\label{eq:pr7}\\
	x_{N}\in\mathcal{X}_{f}\label{eq:pr8}
\end{gather}
where $l_{k}(x_{k},u_{k})$ denotes a cost function to be minimized
such that $l_{k}(x_{k},u_{k})=x_{k}^{\top}Qx_{k}+u_{k}^{\top}Ru_{k}^{\top}$with
$Q=diag(q_{1},q_{2},q_{3})$ and $R=diag(r_{1},r_{2})$  are positive
definite symmetric weight matrices and $V_{N}(x_{N})$ denotes a terminal
cost function such that $V_{N}(x_{N})=x_{N}^{\top}Px_{N}$ with $P=diag(p_{1},p_{2},p_{3})$
is a positive definite symmetric weight matrix. The nonlinear dynamics,
state, and control input constraints, and terminal constraints are described
in \eqref{eq:pr4}-\eqref{eq:pr8}. Given the inherent computational demands associated with nonlinear (MPC), we show in the subsequent section how to substitute the nonlinear constraints in \eqref{eq:pr4} with a cascaded scheme with a linear equivalence model.

\section{Linear Equivalence Model}\label{sec:Sec2_Linear-Equivalence}

\hspace{0.3cm} In this section, the linear equivalence model of the unicycle in \eqref{eq:nonlinearmodel-1} is discussed.
In view of \eqref{lem:Full}, the solvability of the full input to state
feedback linearization of the unicycle model will render the input-output
relation to be linear. The goal of this section is to find the feedback
linearizing controller in the form of $u=\alpha(x)+\beta(x)v$, where
$\alpha(x)$ and $\beta(x)$ can be written as follows:
\begin{equation}
	\alpha(x)=A^{-1}(x)\left[\begin{array}{c}
		\mathcal{L}_{f}^{r_{1}}h_{1}(x)\\
		\vdots\\
		\mathcal{L}_{f}^{r_{m}}h_{m}(x)
	\end{array}\right],\hspace{0.5cm}\beta(x)=A^{-1}(x).\label{eq:200}
\end{equation}

\subsection{Feedback Linearization}

\hspace{0.3cm} The unicycle model in \eqref{eq:nonlinearmodel-1} can be represented as MIMO affine in control system 
as follows:
\begin{equation}
	f(x)=\left[\begin{array}{c}
		0\\
		0\\
		0
	\end{array}\right],g_{1}(x)=\left[\begin{array}{c}
		\cos(x_{3})\\
		\sin(x_{3})\\
		0
	\end{array}\right],g_{2}(x)=\left[\begin{array}{c}
		0\\
		0\\
		1
	\end{array}\right]\label{eq:mobile_3}
\end{equation}
where $h_{1}(x)=x_{1}$ and $h_{2}(x)=x_{2}$. To calculate the relative
degree vector $r=[r_{1},r_{2}]^{\top}$ of \eqref{eq:mobile_3}, recall
\eqref{lem:degree} and consider $r_{1}=1$. One finds
\begin{align}
	\mathcal{L}_{g_{1}}h_{1}= & \left[\begin{array}{ccc}
		1 & 0 & 0\end{array}\right]\left[\begin{array}{c}
		\cos(x_{3})\\
		\sin(x_{3})\\
		0
	\end{array}\right]=\cos(x_{3})\label{eq:mobile_4}\\
	\mathcal{L}_{g_{2}}h_{1}= & \left[\begin{array}{ccc}
		1 & 0 & 0\end{array}\right]\left[\begin{array}{c}
		0\\
		0\\
		1
	\end{array}\right]=0\label{eq:mobile_5}
\end{align}
Since $\mathcal{L}_{g_{1}}h_{1}\neq0$ for $x_{3}\neq\{90,270\}\deg$,
one concludes that $r_{1}=1$. The control input $u_{1}$ appears
in the first output expect for $x_{3}=\{90,270\}$, which is intuitive
since the linear velocity will not contribute to movement in the x direction
when the unicycle has $\theta=\{90,270\}$ (visit Fig. \ref{fig:figure1-1}.(b))
and recall the nonholonomic constraints in \eqref{eq:004}. Consider
$r_{2}=1$, one has
\begin{align}
	\mathcal{L}_{g_{1}}h_{2}= & \left[\begin{array}{ccc}
		0 & 1 & 0\end{array}\right]\left[\begin{array}{c}
		0\\
		\sin(x_{3})\\
		0
	\end{array}\right]=\sin(x_{3})\label{eq:mobile_6}\\
	\mathcal{L}_{g_{2}}h_{2}= & \left[\begin{array}{ccc}
		0 & 1 & 0\end{array}\right]\left[\begin{array}{c}
		0\\
		0\\
		1
	\end{array}\right]=0\label{eq:mobile_7}
\end{align}
Since $\mathcal{L}_{g_{1}}h_{2}\neq0$ for $x_{3}\neq\{0,180\}\deg$,
one concludes that $r_{2}=1$. The control input $u_{2}$ appears
in the second output expect for $x_{3}=\{0,180\}$, which is intuitive
since the linear velocity will not contribute to movement in the y direction
when the unicycle has $\theta=\{0,180\}$ (visit Fig. \ref{fig:figure1-1}.(b))
and recall the nonholonomic constraints in \eqref{eq:004}). In the
view of \eqref{eq:decouplingmatrix-1}, the decoupling input matrix
$A(x)$ can be rewritten as follows:
\begin{equation}
	A(x)=\left[\begin{array}{cc}
		\mathcal{L}_{g_{1}}h_{1} & \mathcal{L}_{g_{2}}h_{1}\\
		\mathcal{L}_{g_{1}}h_{2} & \mathcal{L}_{g_{2}}h_{2}
	\end{array}\right]=\left[\begin{array}{cc}
		\cos(x_{3}) & 0\\
		\sin(x_{3}) & 0
	\end{array}\right]\label{eq:mobile_8-1}
\end{equation}
It becomes apparent that $A(x)$ is singular and in view of \eqref{lem:Full}
the full input-to-state feedback linearization using the control input
\[
u=A^{-1}(x)\left(\left[\begin{array}{c}
	\mathcal{L}_{f}^{r_{1}}h_{1}(x)\\
	\mathcal{L}_{f}^{r_{2}}h_{2}(x)
\end{array}\right]+\left[\begin{array}{c}
	v_{1}\\
	v_{2}
\end{array}\right]\right)
\]
is unsolvable. As such, some modifications are necessary to render
the nonlinear model dynamics in \eqref{eq:nonlinearmodel-1} accounting
for the full input-to-state feedback linearizable form which is the focus of the next subsections. 

\subsection{Dynamic Feedback Linearization (DFL)}

\hspace{0.3cm} DFL also known as the dynamic extension algorithm, is comprehensively
discussed in (\cite{slotine1991applied}, Chapter 6). By analyzing
the decoupling matrix $A(x)$, it becomes apparent that $u_{2}$ is
the main problem as it does not appear in any of the two outputs resulting
in a zero-column which leads to $det(A)=0$. A possible solution is
a decoupling matrix $A(x)$ with $u_{2}$. This can be done by adding
a differential delay to $u_{1}$ where an integrator would allow $u_{2}$
to appear in the $A(x)$ matrix. An integrator in $u_{1}$ adds a
new state $\zeta$ resulting in a new state vector $\bar{x}\in\mathbb{R}^{4}:\bar{x}=\left[x_{1},x_{2},x_{3},\zeta\right]^{\top}$
and new vector fields $\bar{f}(\bar{x})$, $\bar{g}_{1}(\bar{x})$,
and $\bar{g}_{2}(\bar{x})$ as well as new control vector $\mathcal{U}=[\mathcal{U}_{1},\mathcal{U}_{2}]^{\top}$.
The new state $\bar{x}\in\mathbb{R}^{4}:\bar{x}=\left[x_{1},x_{2},x_{3},\zeta\right]^{\top}$
and dynamics can be written as follows: 
\begin{equation}
	\bar{\dot{x}}=\left[\begin{array}{c}
		\zeta\cos(x_{3})\\
		\zeta\sin(x_{3})\\
		\mathcal{U}_{2}\\
		\mathcal{U}_{1}
	\end{array}\right],\hspace{1em}\begin{array}{c}
		y_{1}=x_{1}\\
		y_{2}=x_{2}
	\end{array}\label{eq:mobile10-1}
\end{equation}
In \cite{OrioloIEEETrans}, the authors proved that the extended model
in \eqref{eq:mobile10-1} is feedback linearizable such that the extended
system will be equivalent to a two-chain of double integrators
using the control as follows:

\begin{equation}
	\mathcal{U}=\left[\begin{array}{c}
		v_{1}\cos(x_{3})+v_{2}\sin(x_{3})\\
		-v_{1}\frac{\sin(x_{3})}{\zeta}+v_{2}\frac{\cos(x_{3})}{\zeta}
	\end{array}\right]\label{eq:1101}
\end{equation}

\section{Proposed Scheme}\label{sec:Sec5_Proposed-Scheme}

The key feature of the obstacle avoidance control scheme is the safety
constraints represented by the CBF. In \textcolor{black}{\cite{Aaron2019ECC},}
the authors proposed CBF $\mathcal{H}(x)$ representing a safety metric
as the distance between the moving object and the obstacle. The sufficient
and necessary conditions for the safe maneuvers are based on the usage
of class $\mathcal{K_{\infty}}$ function similar to Lyapunov functions
such that $\mathcal{H}(x,u)\geq\mathcal{-\aleph}(\mathcal{H}(x))\Longleftrightarrow C$
is invariant. The proposed scheme makes use of the CBF concepts in
the MPC formulation cascaded by DFL defined in \eqref{eq:1101}.
The motivation for using the cascaded scheme will be to unlock the
usage of a linear MPC with all its merit compared to the nonlinear
MPC in terms of computational cost and ease of stability guarantees. 
To formulate the MPC on the linear equivalent model rendered by the
DFL is the mapping between the states and control inputs of the original
nonlinear underactuated model of the unicycle and the linear equivalent
the model needs to be presented. 
\begin{lem}
	\label{thm:thm1}(Input and state  mapping of MPC-DFL for
	Unicycle Model) Recall the Unicycle extended model dynamics in \eqref{eq:mobile10-1}.
	Using the control input in \eqref{eq:1101}, the optimization problem
	of the MPC in the cascaded scheme of MPC-DFL can be formulated on
	a linear equivalent model dynamics as $\dot{z}=A_{z}z+B_{z}v$ with
	the following state and input  mapping:
	\begin{equation}
		\begin{aligned}z_{1}= & x_{1}\\
			z_{2}= & \zeta\cos(x_{3})\\
			z_{3}= & x_{2}\\
			z_{4}= & \zeta\sin(x_{3})
		\end{aligned}
		\label{eq:theorem12}
	\end{equation}
	\begin{equation}
		\left[\begin{array}{c}
			v_{1}\\
			v_{2}
		\end{array}\right]=\left[\begin{array}{cc}
			\cos(x_{3}) & -\zeta\sin(x_{3})\\
			\sin(x_{3}) & \zeta\cos(x_{3})
		\end{array}\right]\left[\begin{array}{c}
			\mathcal{U}_{1}\\
			\mathcal{U}_{2}
		\end{array}\right]\label{eq:theorem11}
	\end{equation}
\end{lem}
\begin{proof}
	Recall the extended dynamics in the new MIMO system are also affine
	similar to \eqref{eq:pr2}, where $\bar{f}(\bar{x})$, $\bar{g}_{1}(\bar{x})$,
	and $\bar{g_{2}}(\bar{x})$ can be re-expressed as follows:
	\begin{equation}
		\bar{f}(\bar{x})=\left[\begin{array}{c}
			\zeta\cos(x_{3})\\
			\zeta\sin(x_{3})\\
			0\\
			0
		\end{array}\right],\bar{g}_{1}(\bar{x})=\left[\begin{array}{c}
			0\\
			0\\
			0\\
			1
		\end{array}\right],\bar{g}_{2}(\bar{x})=\left[\begin{array}{c}
			0\\
			0\\
			1\\
			0
		\end{array}\right].\label{eq:mobile11-1}
	\end{equation}
	Note that the new relative degree of the system will be higher due
	to the presence of the integrator. Again the point of departure is
	computing the relative degree. By recalling \eqref{lem:degree} and
	checking if $\bar{r}_{1}=1$, one has
	\begin{equation}
		\mathcal{L}_{\bar{g}_{1}}h_{1}=\left[\begin{array}{cccc}
			1 & 0 & 0 & 0\end{array}\right]\left[\begin{array}{c}
			0\\
			0\\
			0\\
			1
		\end{array}\right]=0\label{eq:mobile12-1}
	\end{equation}
	\begin{equation}
		\mathcal{L}_{\bar{g}_{2}}h_{1}=\left[\begin{array}{cccc}
			1 & 0 & 0 & 0\end{array}\right]\left[\begin{array}{c}
			0\\
			0\\
			1\\
			0
		\end{array}\right]=0\label{eq:mobile13-1}
	\end{equation}
	since $\mathcal{L}_{\bar{g}_{1}}h_{1}=\mathcal{L}_{\bar{g}_{2}}h_{1}=0$,
	one could conclude that $\bar{r}_{1}\neq1$. Now by checking if $\bar{r}_{1}=2$,
	one obtains
	\begin{align}
		\mathcal{L}_{\bar{g}_{1}}\mathcal{L}_{\bar{f}}h_{1}= & \mathcal{L}_{\bar{g_{1}}}(\zeta\cos(x_{3}))\nonumber \\
		= & \left[\begin{array}{cccc}
			0 & 0 & -\zeta\sin(x_{3}) & \cos(x_{3})\end{array}\right]\left[\begin{array}{c}
			0\\
			0\\
			0\\
			1
		\end{array}\right]\nonumber \\
		= & \cos(x_{3})\label{eq:13}
	\end{align}
	where
	\begin{equation}
		\mathcal{L}_{\bar{f}_{1}}h_{1}=\left[\begin{array}{cccc}
			1 & 0 & 0 & 0\end{array}\right]\left[\begin{array}{c}
			\zeta\cos(x_{3})\\
			\zeta\sin(x_{3})\\
			0\\
			0
		\end{array}\right]=\zeta\cos(x_{3})\label{eq:mobile14-1-1}
	\end{equation}
	Since $\mathcal{L}_{\bar{g}_{1}}\mathcal{L}_{\bar{f}}h\neq0$, this
	implies that $\bar{r}_{1}=2$ such that 
	\begin{align}
		\mathcal{L}_{\bar{g}_{2}}\mathcal{L}_{\bar{f}}h_{1}= & \mathcal{L}_{\bar{g_{2}}}(\zeta\cos(x_{3}))\nonumber \\
		= & \left[\begin{array}{cccc}
			0 & 0 & -\zeta\sin(x_{3}) & \cos(x_{3})\end{array}\right]\left[\begin{array}{c}
			0\\
			0\\
			1\\
			0
		\end{array}\right]\nonumber \\
		= & -\zeta\sin(x_{3})\label{eq:16}
	\end{align}
	Considering that $\bar{r}_{2}=1$, one finds
	\begin{equation}
		\mathcal{L}_{\bar{g}_{1}}h_{2}=\left[\begin{array}{cccc}
			0 & 1 & 0 & 0\end{array}\right]\left[\begin{array}{c}
			0\\
			0\\
			0\\
			1
		\end{array}\right]=0\label{eq:mobile17-1}
	\end{equation}
	\begin{equation}
		\mathcal{L}_{\bar{g}_{2}}h_{2}=\left[\begin{array}{cccc}
			0 & 1 & 0 & 0\end{array}\right]\left[\begin{array}{c}
			0\\
			0\\
			1\\
			0
		\end{array}\right]=0\label{eq:mobile18-1}
	\end{equation}
	Since $\mathcal{L}_{\bar{g}_{1}}h_{2}=\mathcal{L}_{\bar{g}_{2}}h_{2}=0$,
	one finds $\bar{r}_{2}\neq1$. Let us check for $\bar{r}_{2}=2$:
	\begin{align}
		\mathcal{L}_{\bar{g}_{1}}\mathcal{L}_{\bar{f}}h_{2}= & \mathcal{L}_{\bar{g_{1}}}(\zeta\sin(x_{3}))\nonumber \\
		= & \left[\begin{array}{cccc}
			0 & 0 & \zeta\cos(x_{3}) & \sin(x_{3})\end{array}\right]\left[\begin{array}{c}
			0\\
			0\\
			0\\
			1
		\end{array}\right]=\sin(x_{3})\label{eq:mobile20-1}
	\end{align}
	where
	\begin{equation}
		\mathcal{L}_{\bar{f}}h_{2}=\left[\begin{array}{cccc}
			0 & 1 & 0 & 0\end{array}\right]\left[\begin{array}{c}
			\zeta\cos(x_{3})\\
			\zeta\sin(x_{3})\\
			0\\
			0
		\end{array}\right]=\zeta\sin(x_{3})\label{eq:mobile-21-2}
	\end{equation}
	Given $\mathcal{L}_{\bar{g}_{1}}\mathcal{L}_{\bar{f}}h\neq0_{1}$,
	it can be conclude that $\bar{r}_{2}=2$ and
	\begin{align}
		\mathcal{L}_{\bar{g}_{2}}\mathcal{L}_{\bar{f}}h_{2}= & \mathcal{L}_{\bar{g_{2}}}(\zeta\sin(x_{3}))\nonumber \\
		= & \left[\begin{array}{cccc}
			0 & 0 & \zeta\cos(x_{3}) & \sin(x_{3})\end{array}\right]\left[\begin{array}{c}
			0\\
			0\\
			1\\
			0
		\end{array}\right]=\zeta\cos(x_{3})\label{eq:mobile-22-3}
	\end{align}
	Therefore, the relative degree has increased to $\bar{r}_{1}=\bar{r}_{2}=2$
	with $\bar{r}_{1}+\bar{r}_{2}=n$. Let us recall \eqref{eq:decouplingmatrix-1}
	and consider \eqref{eq:13}, \eqref{eq:16}, \eqref{eq:mobile20-1},
	\eqref{eq:mobile-22-3}, one obtains the new decoupling matrix and
	it's inverse as follows:
	\begin{equation}
		\bar{A}(x)=\left[\begin{array}{cc}
			\cos(x_{3}) & -\zeta\sin(x_{3})\\
			\sin(x_{3}) & \zeta\cos(x_{3})
		\end{array}\right]\label{eq:moible40}
	\end{equation}
	\begin{equation}
		\bar{A}^{-1}(x)=\left[\begin{array}{cc}
			\cos(x_{3}) & \sin(x_{3})\\
			\frac{-\sin(x_{3})}{\zeta} & \frac{\cos(x_{3})}{\zeta}
		\end{array}\right]\label{eq:mobile41}
	\end{equation}
	where $det(\bar{A}(x))=\zeta\cos^{2}(x)+\zeta\sin^{2}(x)=\zeta$.
	It becomes obvious that $\bar{A}(x)$ is singular only at $\zeta=0$.
	Let us recall \eqref{lem:degree}. The new control input $\mathcal{U}$
	can be written as follows:
	\begin{align}
		\mathcal{U} & =\left[\begin{array}{cc}
			\cos(x_{3}) & \sin(x_{3})\\
			\frac{-\sin(x_{3})}{\zeta} & \frac{\cos(x_{3})}{\zeta}
		\end{array}\right]\left(\left[\begin{array}{c}
			\mathcal{L}_{\bar{f}}^{2}h_{1}(x)\\
			\mathcal{L}_{\bar{f}}^{2}h_{2}(x)
		\end{array}\right]+\left[\begin{array}{c}
			v_{1}\\
			v_{2}
		\end{array}\right]\right)\nonumber \\
		& =\left[\begin{array}{c}
			v_{1}\cos(x_{3})+v_{2}\sin(x_{3})\\
			-v_{1}\frac{\sin(x_{3})}{\zeta}+v_{2}\frac{\cos(x_{3})}{\zeta}
		\end{array}\right]\label{eq:mobile32-1-1}
	\end{align}
	where 
	\begin{align}
		\mathcal{L}_{\bar{f}}^{2}h_{1}(x) & =\mathcal{L}_{\bar{f}}\mathcal{L}_{\bar{f}}h_{1}(x)=\mathcal{L}_{\bar{f}}(\zeta\cos(x_{3}))\nonumber \\
		& =\left[\begin{array}{cccc}
			0 & 0 & -\zeta\sin(x_{3}) & \cos(x_{3})\end{array}\right]\left[\begin{array}{c}
			\zeta\cos(x_{3})\\
			\zeta\sin(x_{3})\\
			0\\
			0
		\end{array}\right]=0\label{eq:mobile28-1}
	\end{align}
	\begin{align}
		\mathcal{L}_{\bar{f}}^{2}h_{2}(x) & =\mathcal{L}_{\bar{f}}\mathcal{L}_{\bar{f}}h_{2}(x)=\mathcal{L}_{\bar{f}}(\zeta\sin(x_{3}))\nonumber \\
		& =\left[\begin{array}{cccc}
			0 & 0 & \zeta\cos(x_{3}) & \sin(x_{3})\end{array}\right]\left[\begin{array}{c}
			\zeta\cos(x_{3})\\
			\zeta\sin(x_{3})\\
			0\\
			0
		\end{array}\right]=0\label{eq:mobile30-1}
	\end{align}
	The proposed change of coordinates for the new system in $z$ coordinates
	represents the outputs given the first derivatives $\mathcal{L}_{\bar{f}}h_{1}$
	and $\mathcal{L}_{\bar{f}}h_{2}.$ Using \eqref{eq:mobile14-1-1},\eqref{eq:mobile-21-2},
	the change coordinates $\Phi(\overline{x})$ can be expressed as follows:
	\begin{equation}
		\Phi(\overline{x})=\left[\begin{array}{c}
			h_{1}\\
			\mathcal{L}_{\bar{f}}h_{1}\\
			h_{2}\\
			\mathcal{L}_{\bar{f}}h_{2}
		\end{array}\right]=\left[\begin{array}{c}
			x_{1}\\
			\zeta\cos(x_{3})\\
			x_{2}\\
			\zeta\sin(x_{3})
		\end{array}\right].\label{eq:mobile34-1}
	\end{equation}
	and the mapping from $z$ to $x$ coordinates is equivalent to
	
	\begin{equation}
		\begin{array}{cc}
			z_{1}= & x_{1}\\
			z_{2}= & \zeta\cos(x_{3})\\
			z_{3}= & x_{2}\\
			z_{4}= & \zeta\sin(x_{3})
		\end{array}\Longleftrightarrow\begin{array}{cc}
			x_{1}= & z_{1}\\
			x_{2}= & z_{3}\\
			x_{3}= & \tan^{-1}(\frac{z_{4}}{z_{2}})\\
			\zeta= & z_{2}\sqrt{1+(\frac{z_{4}}{z_{2}})^{2}}
		\end{array}\label{eq:mobile37-1}
	\end{equation}
	Thereby, the new state equations in $z$ coordinates can be described
	by
	
	\begin{align}
		\dot{z}_{1} & =\zeta\cos(x_{3})=z_{2}\sqrt{1+(\frac{z_{4}}{z_{2}})^{2}}\times\cos(\tan^{-1}(\frac{z_{4}}{z_{2}}))\nonumber \\
		& =z_{2}\sqrt{1+(\frac{z_{4}}{z_{2}})^{2}}\times\frac{1}{\sqrt{1+(\frac{z_{4}}{z_{2}})^{2}}}=z_{2}\label{eq:mobile39-1}
	\end{align}
	
	\begin{align}
		\dot{z}_{2}= & \frac{d}{dt}(\zeta\cos(x_{3}))=\dot{\zeta}\cos(x_{3})-\zeta\sin(x_{3})\dot{x}_{3}\nonumber \\
		= & [v_{1}\cos(x_{3})+v_{2}\sin(x_{3})]\cos(x_{3})\nonumber \\
		& -\zeta[\frac{-v_{1}}{\zeta}\sin(x_{3})+\frac{v_{2}}{\zeta}\cos(x_{3})]\sin(x_{3})\nonumber \\
		= & v_{1}\cos^{2}(x_{3})+v_{1}\sin^{2}(x_{3})=v_{1}\label{eq:mobile45-1}
	\end{align}
	Likewise, $\dot{z}_{3}=z_{4}$ and $\dot{z}_{4}=v_{2}$. Hence, the
	new extended system in $z$ coordinates is a two-chain integrator
	described by
	
	\begin{equation}
		\left[\begin{array}{c}
			\dot{z_{1}}\\
			\dot{z}_{2}\\
			\dot{z}_{3}\\
			\dot{z_{4}}
		\end{array}\right]=\left[\begin{array}{c}
			z_{2}\\
			v_{1}\\
			z_{4}\\
			v_{2}
		\end{array}\right]\label{eq:mobile46-1}
	\end{equation}
\end{proof}
\hspace{-0.3cm}Now, our goal is to propose a cascaded scheme of LMPC with the DFL
controller. The main objective of the Feedback linearizing controller
in \eqref{eq:1101} is to render the nonlinear dynamics of
the unicycle model in \eqref{eq:nonlinearmodel-1} in form of a linear
representation as in \eqref{eq:mobile46-1}, which will lead to faster
online solution and stability guarantees in comparison with the nonlinear
MPC. Fig. \ref{fig:propsedscheme} depicts the proposed  Safety Critical Model Predictive Control based on Dynamic Feedback Linearization (SCMPCDFL) control scheme .

\begin{figure*}[!htbh]
	\begin{centering}
		\includegraphics[width=0.8\textwidth]{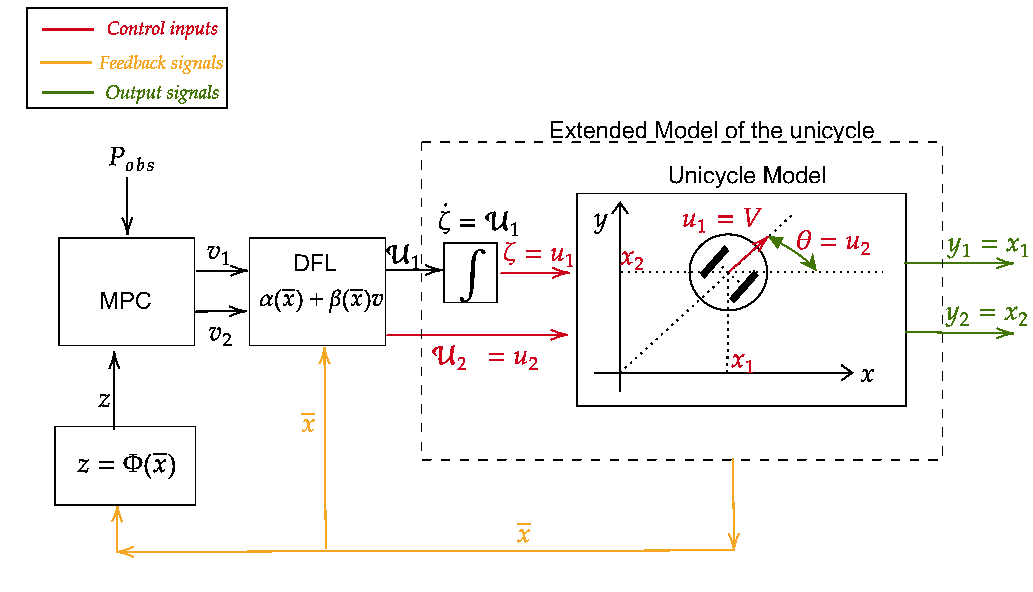}
		\par\end{centering}
	\caption{SCMPCDFL control scheme for the unicycle ground robot.}
	\label{fig:propsedscheme}
\end{figure*}
\hspace{-0.3cm}Let us recall the new extended system in \eqref{eq:mobile46-1} and
discretize the continuous system dynamics using zero-order hold discretization with sampling time $T_{s}$. The following Quadratic
Constraint Quadratic Programming (QCQP) optimization problem can be
formulated:
\begin{align}
	J & =\text{min}_{v_{d}(k)}\sum_{i=0}^{N-1}\left[z_{d}(k+i|k)^{\top}Qz_{d}(k+i|k)\right.\nonumber \\
	& \hspace{1em}\left.+v_{d}(k+i|k)^{\top}Rv_{d}(k+i|k)\right]+z_{d}(N|k)^{\top}\bar{Q}z_{d}(N|k)
	\label{eq:51}
\end{align}
subject to
\begin{gather}
	z_{k+1}=A_{e}z_{k}+B_{e}v_{k},\forall k=0,\ldots.N-1\label{eq:52}\\
	\left[\begin{array}{c} 
		\underbar{\ensuremath{z_{1}}}\\
		\underbar{\ensuremath{z_{3}}}
	\end{array}\right]\leq\left[\begin{array}{c}
		z_{1}\\
		z_{3}
	\end{array}\right]\leq\left[\begin{array}{c}
		\bar{z}_{1}\\
		\bar{z}_{3}
	\end{array}\right],\forall k=0,\ldots.N-1\label{eq:53}\\
	v_{\min}\leq v\leq v_{\max},\forall k=0,\ldots,N-1.\label{eq:531}\\
	\triangle\mathcal{H}(z_{d}(k+1))\geq-\gamma\mathcal{H}(z_{d}(k)),k=0,\ldots,N-1.\label{eq:54}
	\\
	\textcolor{black}{
		v_{\min}\leq K(A_{d}+B_{d}K)^{i}z_{d}(k+N|k)\leq v_{\max},\forall k=0,\ldots,N_{c}.\label{eq:45-2}
	}\\
	\textcolor{black}{
		\left[\begin{array}{c}
			\underline{z}_{1}\\
			\underline{z}_{3}
		\end{array}\right]\leq(A_{d}+B_{d}K)^{i}z_{d}(k+N|k)\leq\left[\begin{array}{c}
			\overline{z}_{1}\\
			\overline{z}_{3} \notag
		\end{array}\right]
	}
	,\\ \forall k=0,\ldots,N_{c}.\label{eq:47}
\end{gather}
where $A_{e}$ and $B_{e}$ denote the discrete system matrices of
\eqref{eq:mobile46-1}, $\underbar{\ensuremath{z_{1}}}$ and $\underbar{\ensuremath{z_{3}}}$
are the minimum value of $z_{1}$, and $z_{3}$, respectively, and
each of $\overline{z}_{1}$ and $\overline{z}_{3}$ refer to the maximum value. The prediction horizon is denoted by $N$ and $N_{c}$ is the constraint
checking horizon.
The proposed control barrier function is defined as follows:

\begin{equation}
	\mathcal{H}(x_{k})=(z_{1}-x_{obs})^{2}+(z_{3}-y_{obs})^{2}-r_{obs}^{2}\label{eq:100}
\end{equation}
such that $x_{obs}$ and $y_{obs}$ describe the $x$ and $y$ coordinates
of the spherical obstacle, respectively and $r_{obs}$ is the radius of the obstacle. It is worth noting that all the constraints
are linear except the Quadratic safety constraint defined in \eqref{eq:54} which is quadratic.
In view of \eqref{eq:54}, one can define the level set of CBF constraints
as follows:

\begin{equation}
	C_{k}=\{x\in\mathbb{R}^{2}:\mathcal{H}(x_{k})=(1-\gamma)\mathcal{H}(x_{k+1})\}\label{eq:101}
\end{equation}
selecting a small value of $\gamma$ will impose a stronger safety
constraint which could potentially lead to unfeasible optimization.
On the other side, selecting a large value of $\gamma$ (e.g., $\gamma=0.99$)
could significantly relax the safety constraint leading to a feasible
solution, however with more risk of having a collision, especially with
a short-sighted prediction $N$. Accordingly, in the next subsection,
the stability of the proposed scheme will be discussed. \textcolor{black}{One
	important remark will come from}\textcolor{red}{{} \eqref{eq:mobile41}, \eqref{eq:mobile32-1-1}}\textcolor{black}{,
	where an expected singularity will happen when $\zeta=0$. One common
	approach to solve the singularity issue is to keep the actual control
	commands bounded or zero when $\zeta$ approaches a certain low value
	donated by $\zeta_{threshold}$ \cite{OrioloIEEETrans}}. Fig \ref{fig:CBF}
visualizes the propsed CBF in \eqref{eq:100} and the level sets in
\eqref{eq:101}.

\begin{figure}[h]
	\begin{centering}
		\includegraphics[scale=0.5]{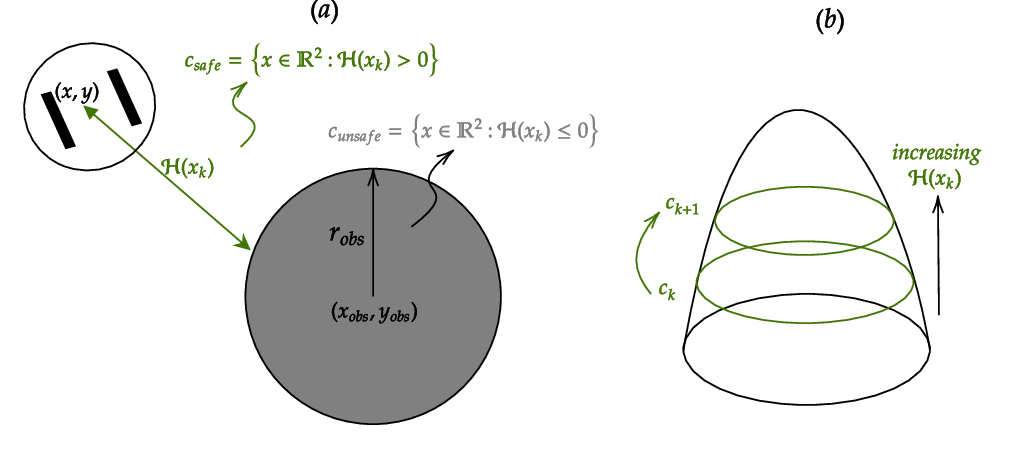}
		\par\end{centering}
	\caption{In (a), $c_{safe}$ and $c_{unsafe}$ are the sets of safe and unsafe
		points in $\mathbb{R}^{3}$ respectively and $\mathcal{H}(x_{k})$
		is the proposed CBF in \eqref{eq:100}. In (b), $c_{k}$ is the level
		set of the CBF defined in \eqref{eq:101}. at a given time $k.$}
	\label{fig:CBF}
\end{figure}

\subsection{Stability and Complexity Analysis}

\hspace{0.3cm} In this work, we consider $\gamma$ as a hyperparameter to be selected
such that the feasibility of the problem holds depending on the number
of obstacles against the robot. Enforcing
a very strong safety constraint could lead to unfeasible solutions.
\begin{assum}\label{claim:The-optimal-predicted} The optimization problem \eqref{eq:51}-\eqref{eq:47} is feasible for the initial time k = 0.\end{assum}

\begin{thm}
	\label{thm:thm2} (Asymptotic Convergence of SCMPCDFL scheme) The extended dynamics of the unicycle model in \eqref{eq:mobile10-1} is asymptotically
	stable and the obstacles are avoided using SCMPCDFL scheme if Assumption
	\ref{claim:The-optimal-predicted} hold true
	and the terminal weight $\bar{Q}$ in \eqref{eq:51} is selected to be equivalent
	to the infinite horizon cost where
	\begin{equation}
		\sum_{i=0}^{\infty}(\Vert z(i)\Vert_{Q}^{2}+\Vert v(i)\Vert_{R}^{2})=z^{\top}(0)\bar{Q}z(0)\label{55-1}
	\end{equation}
	and
	\begin{equation}
		\bar{Q}-(A_{d}+B_{d}K)^{\top}\bar{Q}(A_{d}+B_{d}K)=Q+K^{\top}RK.\label{eq:56-1}
	\end{equation}
	such that the optimization problem in \eqref{eq:51}-\eqref{eq:47} is recursively feasible given the K is stabilizing $(A_{d}+B_{d}K)^{i}$ $\forall i=1,\dots N_{c}$  and $\mathcal{H}((A_{d}+B_{d}K)z_{d})>\mathcal(1-\gamma)\mathcal{H}(z_{d})$ $\forall z_d\in\mathcal{Z}_f$ with sufficiently large $N_{c}$ and $0<\gamma\leq1$.
\end{thm}

\begin{proof}
	Recall the linear equivalent model of the unicycle robot in \eqref{eq:mobile46-1} and the cost function in \eqref{eq:51}. The
	use of the terminal cost function $\bar{Q}$ to solve the Lyapunov
	function in \eqref{eq:56-1} will render the optimal cost function
	for the next time step  $J^{*}(k+1)$ is equal to $J^{*}(k+1)=J^{*}(k)-(\Vert z_{d}(k)\Vert_{Q}^{2}+\Vert v(k)\Vert_{R}^{2})$. 
	Hence, it can be concluded that $J^{*}(k+1)\rightarrow0$ as $k\rightarrow0$. From \eqref{eq:mobile37-1},
	one finds as $z\rightarrow0$ and $\bar{x}\rightarrow0$ and this proves the asymptotic convergence. We prove the recursive stability using the classical terminal constraints. 
	Let $\tilde{v}_{d}(k+1)$ denote the input sequence at time $k+1$ corresponding to the optimal prediction
	at time $k$. For feasible $\tilde{v}_{d}(k+1)$ the $N$th element (the tail of $v_{d}^{*}(k+N|k)=Kz_{d}^{*}(k+N|k)$ is required to satisfy the terminal constraint. This is equivalent to the constraints on the terminal state prediction $z_{d}(k+N|k)\in\mathcal{Z}_{f}$, where $\mathcal{Z}_{f}$ is the terminal set \cite{scokaert1996infinite}. The necessary and sufficient conditions for the predictions generated by the tail $\tilde{v}_{d}(k+1)$ are feasible at time $k+1$ is to have $\mathcal{Z}_{f}$ is control and safe invariant. The terminal set $\mathcal{Z}_{f}$ is control invariant  if
	$(A_{d}+B_{d}K)z_{d}(k+N|k)\in\text{\ensuremath{\mathcal{Z}_{f}}\hspace{0.1cm}\ensuremath{\forall z_{d}(k+N|k)\in\mathcal{Z}_{f}}}.  
	$ The terminal set $\mathcal{Z}_{f}$ is safe invariant if 
	$
	\triangle\mathcal{H}(z_{d}(k+1|N))\geq-\gamma\mathcal{H}(z_{d}(k|N))\text{\ensuremath{}\hspace{0.1cm}\ensuremath{\forall z_{d}(k+N|k)\in\mathcal{Z}_{f}.}}
	$
	To render $\mathcal{Z}_{f}$ control invariant, we
	need to verify that
	\begin{flalign}
		v_{\text{max}} & \leq K(A_{d}+B_{d}K)^{i}z_{d}(k+N|k)\leq v_{\text{max}}\label{eq:62}\\
		\left[\begin{array}{c}
			\underline{z}_{1}\\
			\underline{z}_{3}\\
		\end{array}\right]\leq & (A_{d}+B_{d}K)^{i}z_{d}(k+N|k)\leq\left[\begin{array}{c}
			\overline{z}_{1}\\
			\overline{z}_{3}\\
		\end{array}\right]
		& \forall i\geq0\label{eq:64}
	\end{flalign}
	one can design $K$ to render $(A_{d}+B_{d}K)^{i}$ stable such that the
	norm of $|\lambda(A_{d}+B_{d}k)|<1.$ To render $\mathcal{Z}_{f}$ safe invariant, the constraint \eqref{eq:54} will impose that $\mathcal{H}(z_{d}(k+1|N))-\mathcal{H}(z_{d}(k||N))\geq-\gamma\mathcal{H}(z_{d}(k|N))$
	and as a result 
	with $0<\gamma\leq1$. We have $\mathcal{H}(z_{d}(k+1|N))>\mathcal{H}(z_{d}(k|N)$. From \eqref{eq:54}, if $z_{d}(N|0)\in\mathcal{Z}_{f}$
	such that $\mathcal{H}(z_{d}(N|0)\geq0$. One can design $K$ such that   $\mathcal{H}((A_{d}+B_{d}K)z_{d})>\mathcal(1-\gamma)\mathcal{H}(z_{d})$ with $0<\gamma\leq1$ proving the safety invariance of $\mathcal{Z}_{f}$.
\end{proof}

Assumption  \ref{claim:The-optimal-predicted}
provides the initial feasibility  conditions at $k=0$ where  Theorem \ref{thm:thm2}  guarantees the feasibility for the next time steps.
We can define the terminal constraint set as
$
{Z}_{f}(N_{c})=\{z_{d}:v_{\text{min}}\leq K(A_{d}+B_{d}K)^{i}z_{d}\leq v_{\text{max}},i=0,1,\ldots N_{c}\}\
$.
By choosing sufficiently large $N_{c}$, the allowable operating region of the MPC law will be increased satisfying all the constraints.
Note that the constraints in \eqref{eq:45-2} and \eqref{eq:47} are extra computational burden specially with large $N_{c}$ which is needed to maintain the feasibility (visit Theorem \ref{thm:thm2}). However, the constraints in \eqref{eq:45-2} and \eqref{eq:47} are linear constraints, where the QCQP presented in \eqref{eq:51}-\eqref{eq:47} can be solved efficiently by off-the-shelf solvers. Algorithm \ref{algo} summarizes the implementation steps of the proposed control scheme in Fig.\ref{fig:propsedscheme}. Given the known place and shape of the obstacle and for every time step the  starts by measuring the states and makes use of the coordinate transformation $\Phi(\overline{x})$ 
representing Lemma \ref{thm:thm1} to compute the equivalent z states.Then the QCQP problem mentioned in \eqref{eq:51}-\eqref{eq:47} will be solved to compute the optimal value of $v$  that will be given the DFL controller to compute the control inputs that will result in an obstacle-free path. The QCQP can be solved by Interior
Point OPTimizer (IPOPT) in MATLAB efficiently, and the computational complexity can be estimated by expected number of flops (floating point operations). The worst-case number of flops for the QCQP can be approximated by \cite{richter2011computational} as follows:
\begin{equation}    
	\#\text { flops }_{\mathrm{IP}}=i_{\mathrm{IP}}\left(2 / 3(\mathrm{Nm})^3+2(\mathrm{Nm})^2\right)
	\label{eq:flops}
\end{equation}
where $N$ is the prediction horizon and $m$ is the number of control inputs. The $i_{\mathrm{IP}}$ is the number of IP iterations which is expected to be $\mathcal{O}(\sqrt{Nm}\log(1/\epsilon))$ for an $\epsilon$-accurate solution \cite{SHEN2020108863}. 

In contrast, if the problem is formulated as an NMPC problem, the associated nonlinear problem is way more expensive to solve. Assuming that we solve it using the SQP method, omitting minor operations in function evaluations, gradient and Hessian updating, and line search in one major SQP iteration, a rough estimate of the computational complexity is
\begin{equation}
	\#\text { flops }_{\mathrm{SQP}} \approx i_{\mathrm{SQP}} \times\left(\# \text { flops }_{\mathrm{IP}}\right)
	\label{eq:SQP_flops}
\end{equation}
where $i_{\mathrm{SQP}}$ denotes the numbers of SQP iterations required for convergence or it can be specified by the maximum iteration number $i_{\mathrm{SQP}, \max }$. From (\ref{eq:SQP_flops})  readily, we have seen the advantage of the LMPC formulation. In the next section, we will show how the proposed scheme will lead to shorter prediction horizons $N$ further reducing effectively the computational complexity recalling \eqref{eq:flops}.

\begin{algorithm}[h]
	\caption{Safety Critical MPC based on DFL (SCMPCDFL)}
	
	\textbf{Inputs}:
	\begin{enumerate}
		\item $x(0)$, $x_{pos}$ and $y_{pos}$
	\end{enumerate}
	\textbf{Outputs}:
	\begin{enumerate}
		\item Control Signals $u_{1}$, $u_{2}$
	\end{enumerate}
	\textbf{For ever time step do}
	\begin{enumerate}
		\item[{\footnotesize{}1:}]  Measure the current state $\bar{x}_{k}$
		\item[{\footnotesize{}2:}]  Compute $z_{k}$ using the mapping in \eqref{eq:mobile37-1}
		\item[{\footnotesize{}3:}]  Solve the optimization problem in \eqref{eq:51}-\eqref{eq:47}
		and get $\{v^{*}(k+1|k),v^{*}(k+2|k),\ldots,Kx^{*}(k+N|k)\}$
		\item[{\footnotesize{}4:}]  Apply the first control input $v=v^{*}(k+1|k)$
		\item[{\footnotesize{}5:}]  Compute the control signals $u_{1}=v_{1}\cos(x_{3})+v_{2}\sin(x_{3})$
		and $u_{2}$ as $\begin{cases}
			u_{2}=-v_{1}\frac{\sin(x_{3})}{\zeta}+v_{2}\frac{\cos(x_{3})}{\zeta} & \zeta>\zeta_{threshold}\\
			u_{2}=0 & \zeta\leq\zeta_{threshold}
		\end{cases}.$
	\end{enumerate}
	\textbf{end }For
	\label{algo}
\end{algorithm}

\section{Results}\label{sec:Sec6_Results}
\begin{figure*}[!htb]
	\centering
	\begin{subfigure}[b]{0.475\textwidth}
		\centering
		\includegraphics[width=9cm,height=6cm]{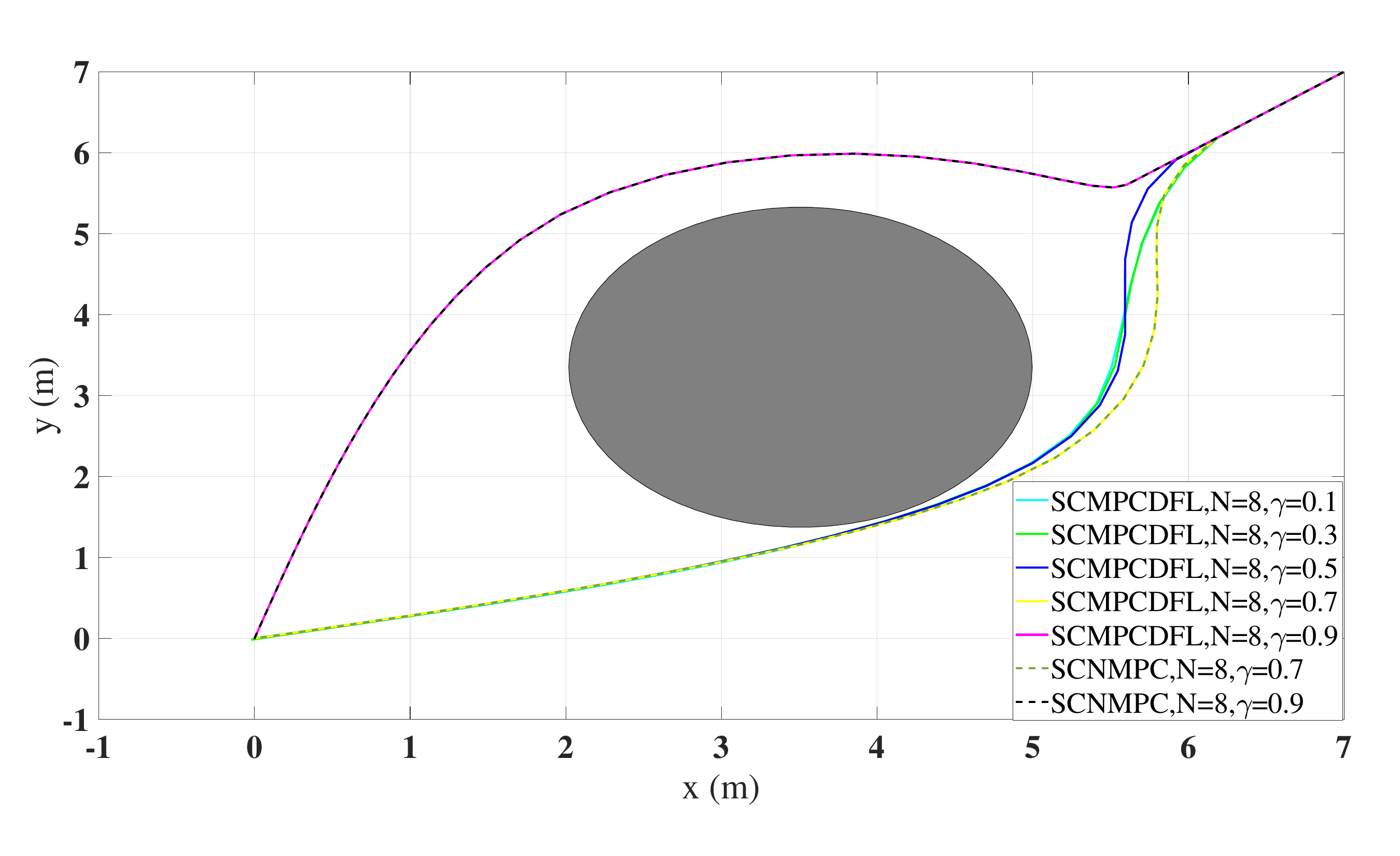}
		\caption[Network2]%
		{{}}    
		\label{}
	\end{subfigure}
	\hfill
	\begin{subfigure}[b]{0.475\textwidth}  
		\centering 
		\includegraphics[width=9cm,height=6.1cm]{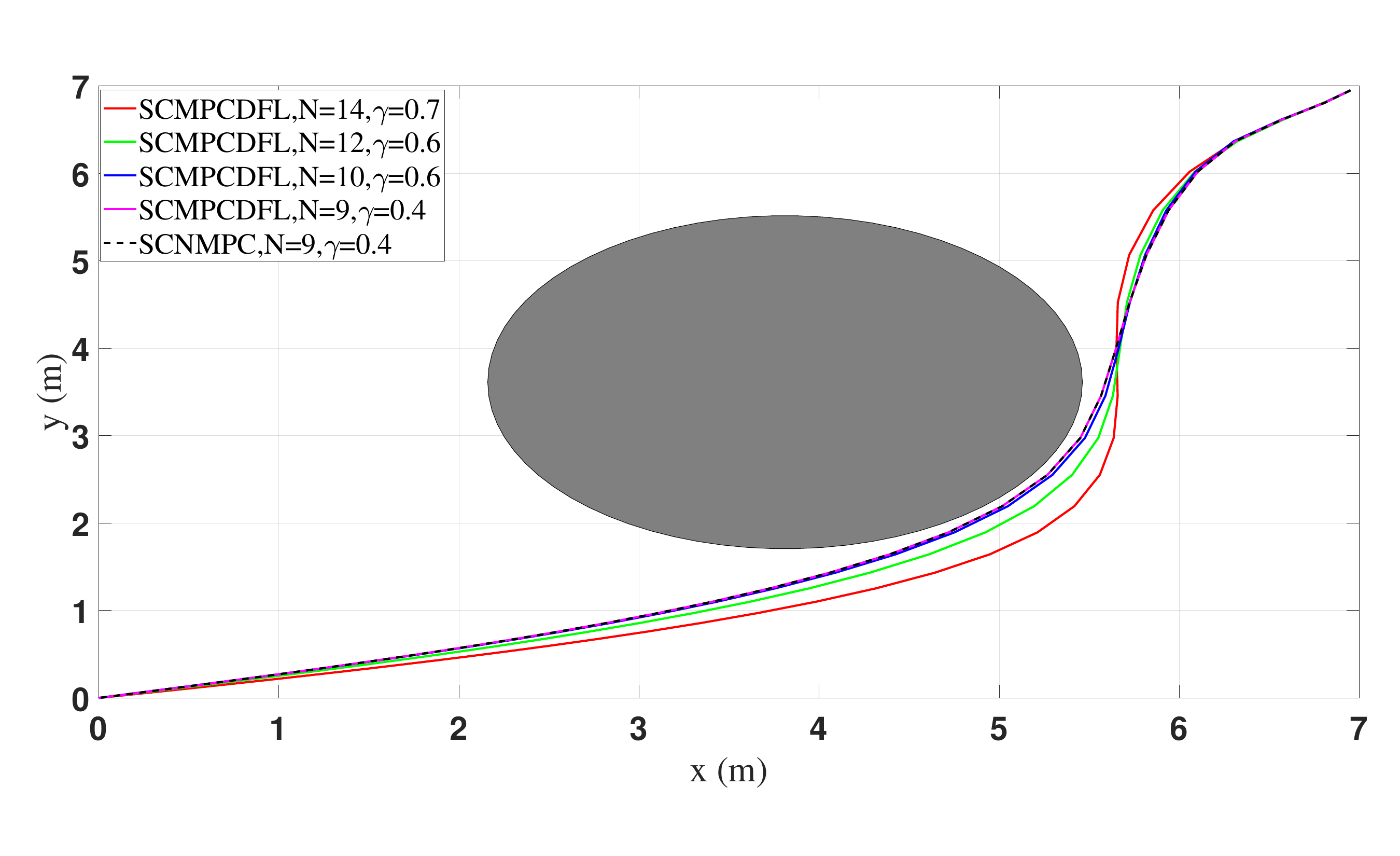}
		\caption[]%
		{{}}    
		\label{}
	\end{subfigure}
	\vskip\baselineskip
	\begin{subfigure}[b]{0.475\textwidth}   
		\centering 
		\includegraphics[width=9cm,height=6cm]{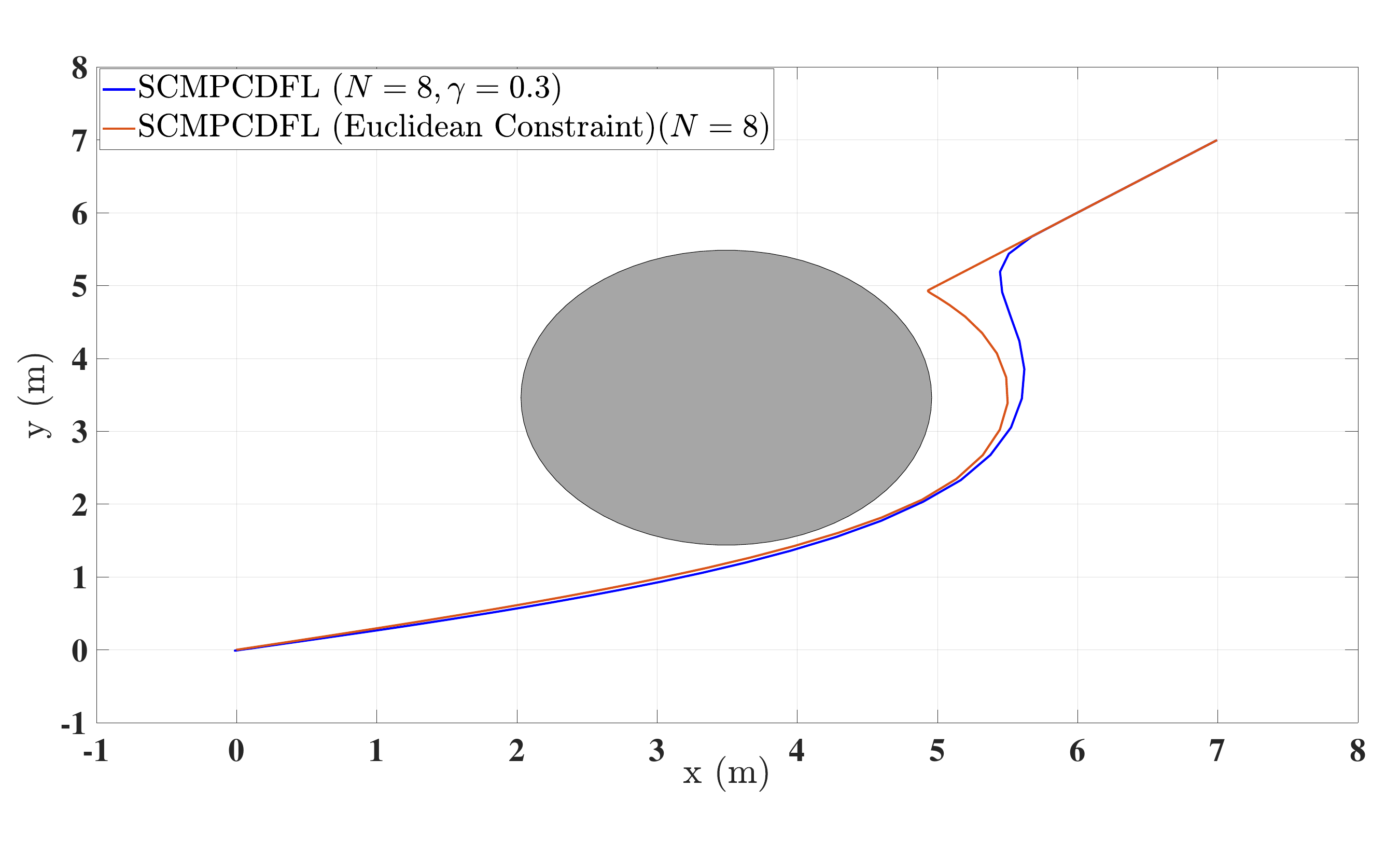}
		\caption[]%
		{{}}    
		\label{}
	\end{subfigure}
	\hfill
	\begin{subfigure}[b]{0.475\textwidth}   
		\centering 
		\includegraphics[width=9cm,height=6cm]{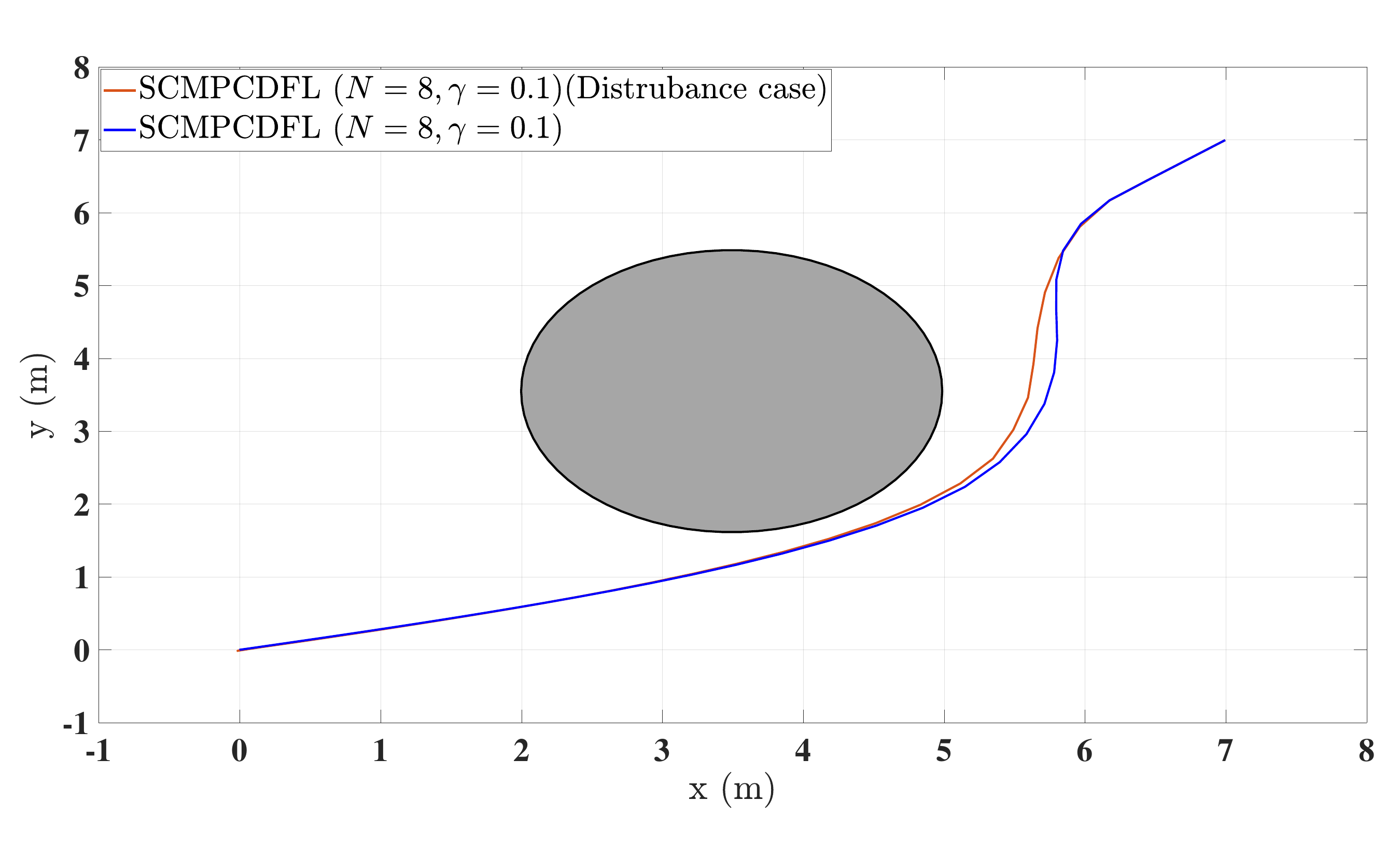}
		\caption[]%
		{{}}    
		\label{}
	\end{subfigure}
	\caption{Part (a) shows the Output performance of the SCMPCDFL and SCNMPC Control (presented in \eqref{eq:pr3}-\eqref{eq:pr7}) schemes with different values of  $\gamma=0.1,0.3,0.5,0.7,0.9,1$ and $N=8$. Part (b) illustrates the performance with different values of  $\gamma$ and $N$. Part (c) compares between the  performance of the proposed SCMPCDFL utilizing CBF against the benchmark of using Euclidean distance with the same prediction horizon $N=8$. Part (D) depicts a comparison between the output performance of the proposed scheme (demonstrated in Blue line) with and without output distance by a Gaussian noise of zero mean and 0.05 variance.}
	
	\label{fig:Result1}
\end{figure*}

\begin{figure}[!htb]
	\begin{centering}
		\includegraphics[width=0.5\textwidth]{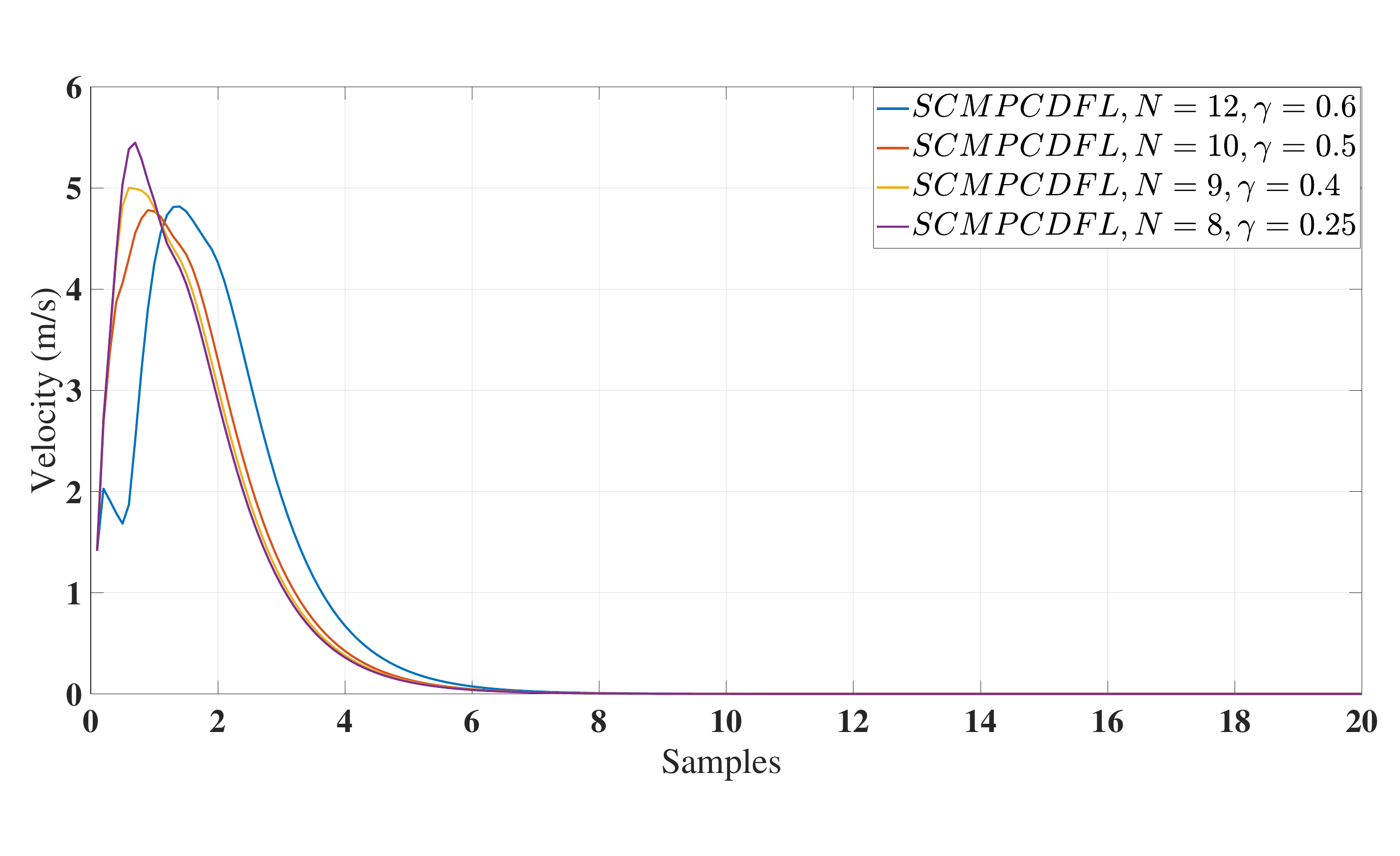}
		\par\end{centering}
	\caption{The linear velocity of the robot with different values of $N$
		and $\gamma.$ }
	\label{fig:vel}
\end{figure}
The computing source used to carry out simulations in this section
has the following specifications: Windows 11 (64-bit processor), CPU
(Intel(R) Core(TM) i7-10750H, 2.60GHz, 2592 MHz, 6 core, and 12 logic
processors), and 16 GB RAM. The QCQP problem in \eqref{eq:51}-\eqref{eq:47}
was solved using IPOPT method in MATLAB and the
continuous dynamics and DFL controller in \eqref{eq:nonlinearmodel-1}, \eqref{eq:1101}
was numerically integrated using the 4th Order Runge Kutta Method with
the same sampling time $T_{s}$. The proposed scheme is tested with
a sampling time $T_{s}=0.05$ sec and $\zeta_{threshold}=0.01$. The
robot starts from the initial position $(x=7,y=7)$ meters and
then goes to the origin against a circular obstacle placed in the
middle of the map with a radius of 1.5 meters.
Fig. \ref{fig:Result1}.(a)
shows the effect of $\gamma$ on the prediction horizon in the evolution
of the robot. The value of $\gamma$ plays a role in imposing stronger or weaker safety constraints. In the case of $\gamma=0.1$ (the green
line in Fig. \ref{fig:Result1}.(a)), the robot deviates far from the obstacle, since $\gamma=0.1$
imposes strong safety constraints with respect to the case of $\gamma=0.3,0.5,0.7,0.9$.
In the case of $\gamma=1$, the robot almost colloids with
the obstacle.
Fig. \ref{fig:Result1}.(b) illustrates the merit of using the CBF, Where a strong safety constraint using a small value of $\gamma$ unlocks the possibility of using a smaller prediction horizon. Even though the small prediction horizon $N=9$ in the red line in Fig. \ref{fig:Result1}.(c), it has the safest path compared to higher prediction horizons $N=10,12,14$ with the bigger value of $\gamma$. Fig. \ref{fig:Result1}.(c) demonstrates another benefit of using CBF compared to the benchmark of the Euclidean distance constraint that will be only activated when the robot is near the obstacle. The proposed scheme which is based on CBF (demonstrated in the blue line) starts to deviate earlier from the case of Euclidean distance constraint having the same prediction horizon $N=8$. Since the proposed scheme relies on the measured values of the states, the robustness of the scheme is tested in Fig. \ref{fig:Result1}.(d) against noisy feedback corrupted with Gaussian noise with zero mean and 0.05 variance. The case of corrupted feedback (demonstrated in the red line) can still avoid the obstacle with some deterioration compared to the nominal case. Fig. \ref{fig:vel} depicts the linear velocity of the robot with
different values of $N$ and $\gamma$.



The distance between the robot and the obstacle is shown in Fig. \ref{distancefigure} using different values of $\gamma$ and $N$ in both cases of using the Euclidean distance constraint and the proposed scheme.
The real-time implementation
of the QCQP problem needs to be solved within strict time constraints
with respect to the sampling time. \textcolor{black}{Violation of
	the time constraints could degrade the output performance
	and/or stability measures}. To verify that, Fig. \ref{fig:comptime}
shows the computational time for solving the QCQP problem, where the computational time doesn't exceed the sampling time $T_{s}=0.05$. 
Finally in Fig. \ref{fig:convergence}, the convergence of the states is shown, where the asymptotic convergence in Theorem \ref{thm:thm2} is verified. In case of disturbance(demonstrated in dashed lines in Fig. \ref{fig:convergence}) the states converge to a neighborhood of the origin.
		


\begin{figure}[!htb]
	\begin{centering}
		\includegraphics[width=9cm,height=6cm]{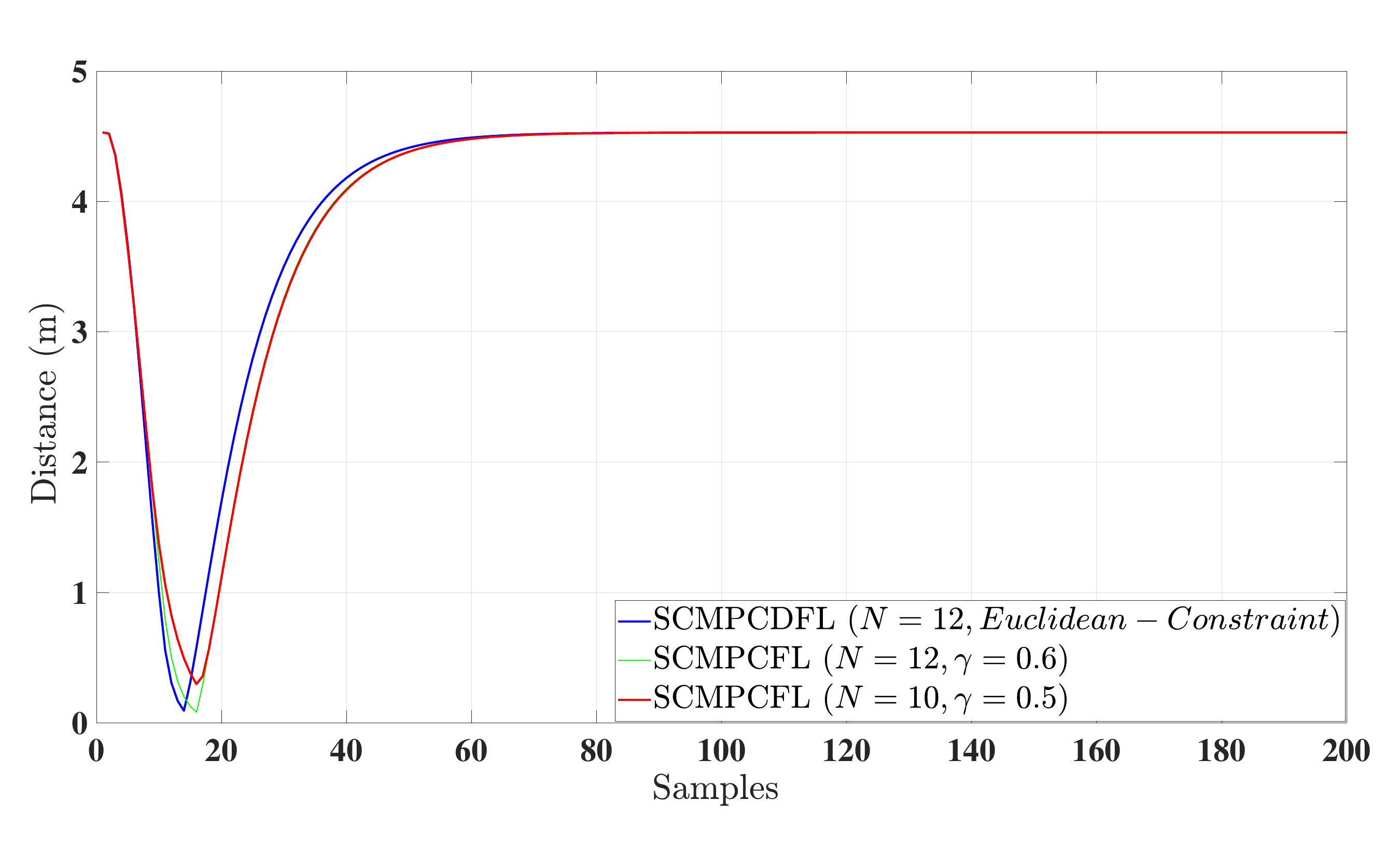}
		\par\end{centering}
	\caption{A comparison of the distance between the robot between the proposed scheme SCMPCDFL utilizing CBF against the benchmark of using Euclidean distance. }
	\label{distancefigure}
\end{figure}

\begin{figure}[!htb]
	\begin{centering}
		\includegraphics[width=9cm,height=6cm]{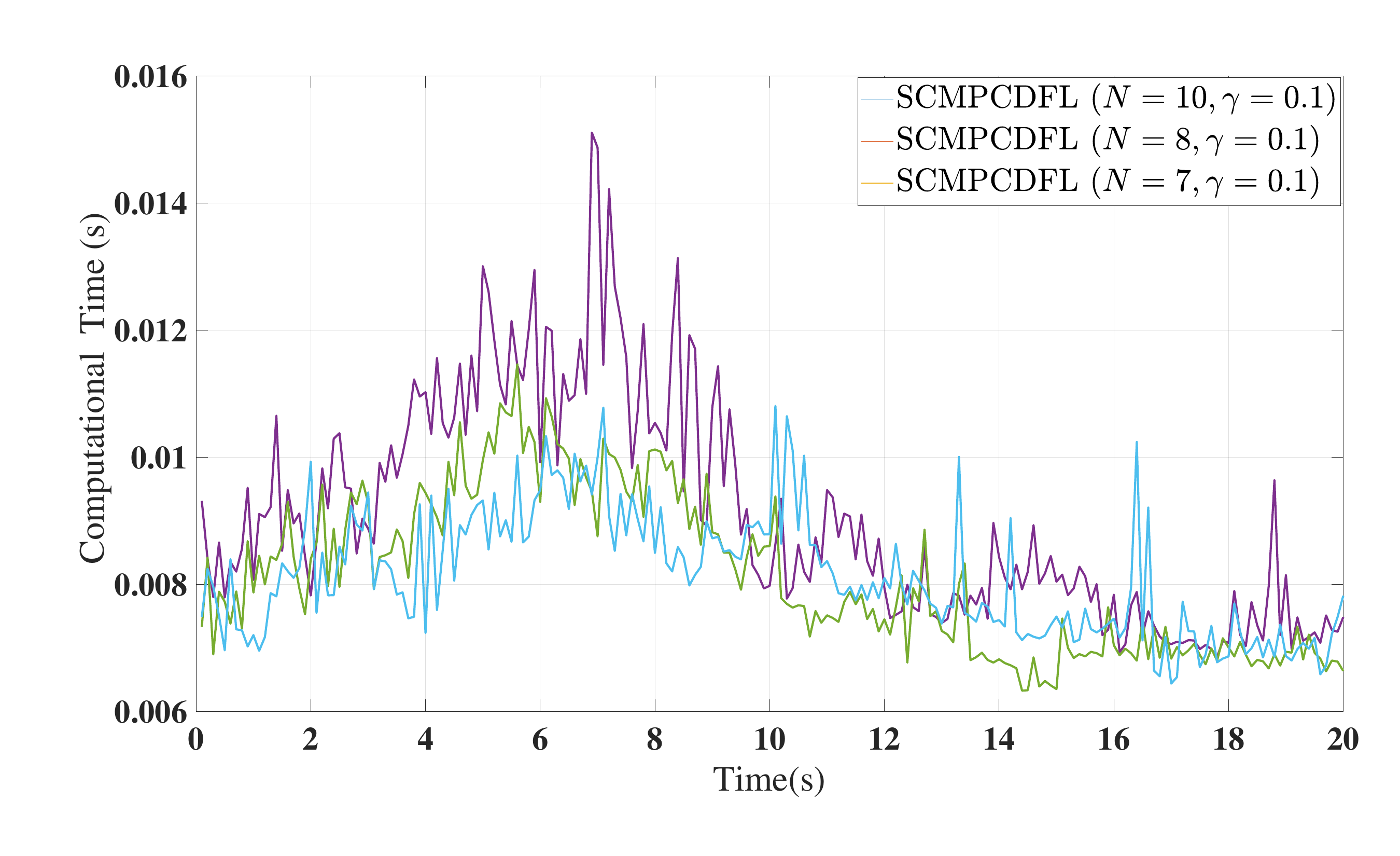}
		\par\end{centering}
	\caption{A comparison between computational time of SCMPCDFL and SCNMPC (presented in \eqref{eq:pr3}-\eqref{eq:pr7}) schemes time in seconds with different prediction horizon
		with $\gamma=0.1$.}
	\label{fig:comptime}
\end{figure}

\begin{figure}[!htb]
	\begin{centering}
		\includegraphics[width=9cm,height=6cm]{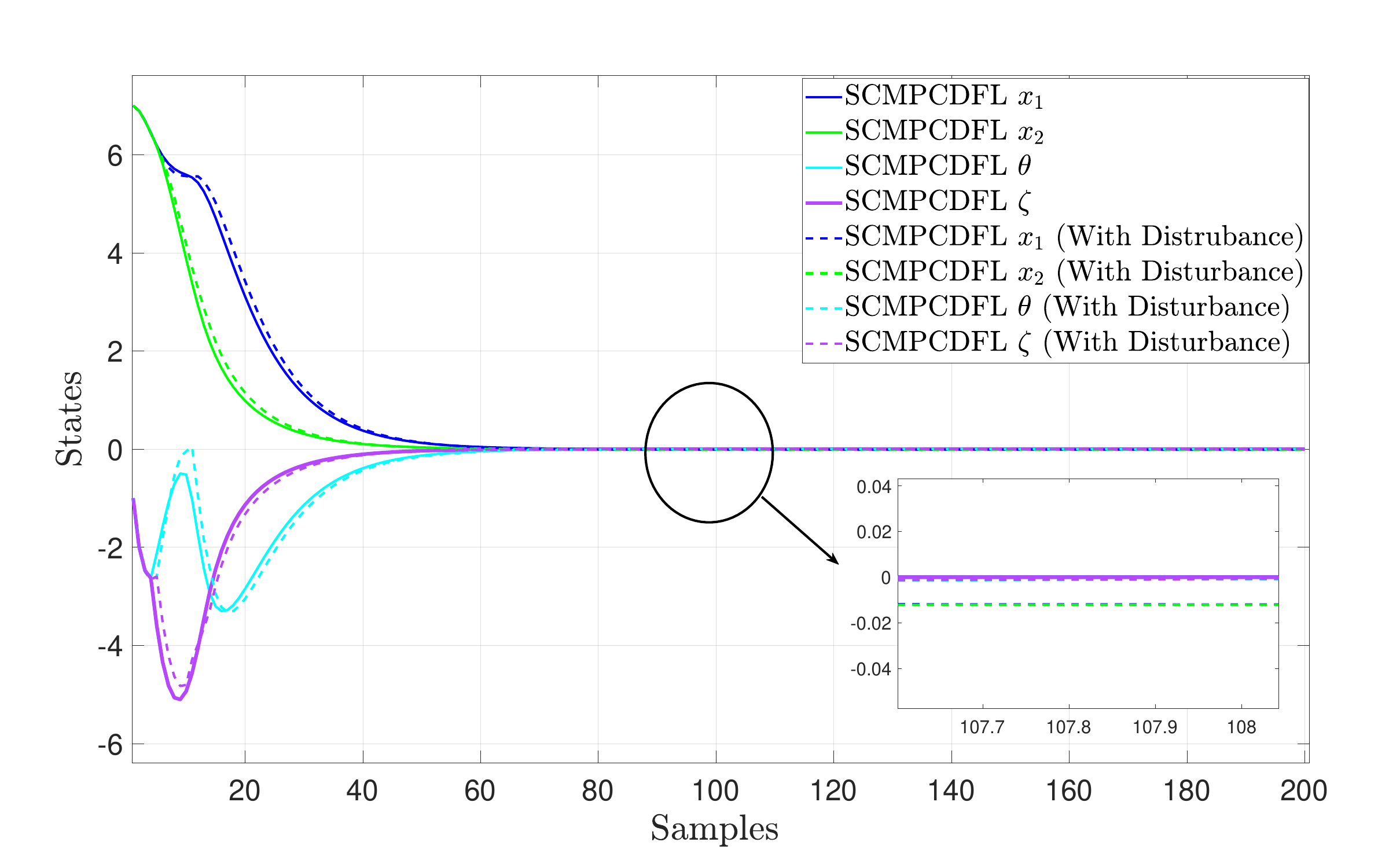}
		\par\end{centering}
	\caption{The convergence of extended state of the robot $\bar{x}$ to the origin with and without disturbance using $\gamma=0.1$ and $N=8$.}
	\label{fig:convergence}
\end{figure}

\section{Conclusion}\label{sec:Sec7_Conclusion}

A Safety-Critical Model Predictive Control Design (SCMPC) based on linear equivalence tailored to the differential mobile robots with two wheels is proposed in this paper. A cascaded scheme of Model Predictive Control (MPC) controller along with a Dynamic Feedback Linearization (DFL) controller is presented to achieve safe navigation of the unicycle ground robot. The proposed scheme shows a strong theoretical guarantee for closed-loop stability and recursive feasibility based on Linear MPC (LMPC) analysis. Numerical simulations show successful obstacle-free navigation paths where fast and easy implementation can be achieved with a theoretical guarantee of convergence. The selected $\gamma$ is crucial for imposing stronger or weaker safety constraints. The utilization of the Control Barrier Functions (CBF) within the safety constraints improves the obstacle avoidance maneuvers with respect to the constraints based on Euclidean distances. Numerical results illustrated the robustness of the proposal. Future work would be investigating the recursive feasibility and performing experimental experiments.


\section{Acknowledgments}\label{sec11}

This work was supported in part by the National Sciences and Engineering Research Council of Canada (NSERC), under the grants RGPIN-2022-04937, RGPIN-2022-04940, DGECR-2022-00103 and DGECR-2022-00106.

\balance
\bibliographystyle{IEEEtran}
\bibliography{Bib_MPC}

\begin{thebibliography}{10}
\providecommand{\url}[1]{#1}
\csname url@samestyle\endcsname
\providecommand{\newblock}{\relax}
\providecommand{\bibinfo}[2]{#2}
\providecommand{\BIBentrySTDinterwordspacing}{\spaceskip=0pt\relax}
\providecommand{\BIBentryALTinterwordstretchfactor}{4}
\providecommand{\BIBentryALTinterwordspacing}{\spaceskip=\fontdimen2\font plus
\BIBentryALTinterwordstretchfactor\fontdimen3\font minus
  \fontdimen4\font\relax}
\providecommand{\BIBforeignlanguage}[2]{{%
\expandafter\ifx\csname l@#1\endcsname\relax
\typeout{** WARNING: IEEEtran.bst: No hyphenation pattern has been}%
\typeout{** loaded for the language `#1'. Using the pattern for}%
\typeout{** the default language instead.}%
\else
\language=\csname l@#1\endcsname
\fi
#2}}
\providecommand{\BIBdecl}{\relax}
\BIBdecl

\bibitem{AaronIEEE}
A.~D. Ames, X.~Xu, J.~W. Grizzle, and P.~Tabuada, ``Control barrier function
  based quadratic programs for safety critical systems,'' \emph{IEEE
  Transactions on Automatic Control}, vol.~62, no.~8, pp. 3861--3876, 2017.

\bibitem{IET-Trakinghuman}
T.-H. Tsai and C.-H. Yao, ``A robust tracking algorithm for a human-following
  mobile robot,'' \emph{IET Image Processing}, vol.~15, no.~3, pp. 786--796,
  2021.

\bibitem{KanayamaICRA}
Y.~Kanayama, Y.~Kimura, F.~Miyazaki, and T.~Noguchi, ``A stable tracking
  control method for an autonomous mobile robot,'' in \emph{Proceedings., IEEE
  International Conference on Robotics and Automation}, 1990, pp. 384--389
  vol.1.

\bibitem{OrioloIEEETrans}
G.~Oriolo, A.~De~Luca, and M.~Vendittelli, ``Wmr control via dynamic feedback
  linearization: design, implementation, and experimental validation,''
  \emph{IEEE Transactions on Control Systems Technology}, vol.~10, no.~6, pp.
  835--852, 2002.

\bibitem{SaturationfeedbackIEEETrans}
T.-C. Lee, K.-T. Song, C.-H. Lee, and C.-C. Teng, ``Tracking control of
  unicycle-modeled mobile robots using a saturation feedback controller,''
  \emph{IEEE Transactions on Control Systems Technology}, vol.~9, no.~2, pp.
  305--318, 2001.

\bibitem{IntegralSlidingFridmanIEEETrans}
M.~Rubagotti, A.~Estrada, F.~Castanos, A.~Ferrara, and L.~Fridman, ``Integral
  sliding mode control for nonlinear systems with matched and unmatched
  perturbations,'' \emph{IEEE Transactions on Automatic Control}, vol.~56,
  no.~11, pp. 2699--2704, 2011.

\bibitem{ferraraIETControlTheory}
A.~Ferrara and M.~Rubagotti, ``Second-order sliding-mode control of a mobile
  robot based on a harmonic potential field,'' \emph{IET Control Theory \&
  Applications}, vol.~2, no.~9, pp. 807--818, 2008.

\bibitem{hwang2013trajectoryIET}
C.-L. Hwang and H.-M. Wu, ``Trajectory tracking of a mobile robot with
  frictions and uncertainties using hierarchical sliding-mode under-actuated
  control,'' \emph{IET Control Theory \& Applications}, vol.~7, no.~7, pp.
  952--965, 2013.

\bibitem{IETUnicycleFintetime}
M.~Liu, K.~Wu, and Y.~Wu, ``Finite-time tracking control of disturbed
  non-holonomic systems with input saturation and state constraints: Theory and
  experiment,'' \emph{IET Control Theory \& Applications}, vol.~8, no.~2, 2023.

\bibitem{IETSliding2010}
S.~Yoo, ``Adaptive tracking control for a class of wheeled mobile robots with
  unknown skidding and slipping,'' \emph{IET Control Theory \& Applications},
  vol.~4, no.~10, pp. 2109--2119, 2010.

\bibitem{IETprescribedperformance2015}
M.~Zambelli, Y.~Karayiannidis, and D.~V. Dimarogonas, ``Posture regulation for
  unicycle-like robots with prescribed performance guarantees,'' \emph{IET
  Control Theory \& Applications}, vol.~9, no.~2, pp. 192--202, 2015.

\bibitem{garcia1989bookmpc}
C.~E. Garcia, D.~M. Prett, and M.~Morari, ``Model predictive control: Theory
  and practice—a survey,'' \emph{Automatica}, vol.~25, no.~3, pp. 335--348,
  1989.

\bibitem{bookoptimization}
T.~Faulwasser, \emph{Optimization-based solutions to constrained
  trajectory-tracking and path-following problems}.\hskip 1em plus 0.5em minus
  0.4em\relax Otto-von-Guericke Universität Magdeburg und Institut für
  Automatisierungstechnik Magdeburg, 01 2013.

\bibitem{MPCmobileIFAC2018}
S.~Yu, Y.~Guo, L.~Meng, T.~Qu, and H.~Chen, ``Mpc for path following problems
  of wheeled mobile robots,'' \emph{IFAC-PapersOnLine}, vol.~51, no.~20, pp.
  247--252, 2018.

\bibitem{SunIEEETrans2017}
Z.~Sun, Y.~Xia, L.~Dai, K.~Liu, and D.~Ma, ``Disturbance rejection mpc for
  tracking of wheeled mobile robot,'' \emph{IEEE/ASME Transactions on
  Mechatronics}, vol.~22, no.~6, pp. 2576--2587, 2017.

\bibitem{Aaron2019ECC}
A.~D. Ames, S.~Coogan, M.~Egerstedt, G.~Notomista, K.~Sreenath, and P.~Tabuada,
  ``Control barrier functions: Theory and applications,'' in \emph{2019 18th
  European Control Conference (ECC)}, 2019, pp. 3420--3431.

\bibitem{Ali2024ACC}
A.~M. Ali, H.~A. Hashim, and C.~Shen, ``Mpc based linear equivalence with
  control barrier functions for vtol-uavs,'' 2024, pp. 1--6.

\bibitem{YOON2009741}
Y.~Yoon, J.~Shin, H.~J. Kim, Y.~Park, and S.~Sastry, ``Model-predictive active
  steering and obstacle avoidance for autonomous ground vehicles,''
  \emph{Control Engineering Practice}, vol.~17, no.~7, pp. 741--750, 2009.

\bibitem{6728261}
V.~Turri, A.~Carvalho, H.~E. Tseng, K.~H. Johansson, and F.~Borrelli, ``Linear
  model predictive control for lane keeping and obstacle avoidance on low
  curvature roads,'' in \emph{16th International IEEE Conference on Intelligent
  Transportation Systems (ITSC 2013)}, 2013, pp. 378--383.

\bibitem{7489011}
U.~Rosolia, S.~De~Bruyne, and A.~G. Alleyne, ``Autonomous vehicle control: A
  nonconvex approach for obstacle avoidance,'' \emph{IEEE Transactions on
  Control Systems Technology}, vol.~25, no.~2, pp. 469--484, 2017.

\bibitem{Aaron2017IEEETrans}
A.~D. Ames, X.~Xu, J.~W. Grizzle, and P.~Tabuada, ``Control barrier function
  based quadratic programs for safety critical systems,'' \emph{IEEE
  Transactions on Automatic Control}, vol.~62, no.~8, pp. 3861--3876, 2017.

\bibitem{Grandia2020NonlinearMP}
R.~Grandia, A.~J. Taylor, A.~W. Singletary, M.~Hutter, and A.~Ames, ``Nonlinear
  model predictive control of robotic systems with control lyapunov
  functions,'' \emph{ArXiv}, vol.~50, 2020.

\bibitem{SonCDC2019}
T.~D. Son and Q.~Nguyen, ``Safety-critical control for non-affine nonlinear
  systems with application on autonomous vehicle,'' in \emph{2019 IEEE 58th
  Conference on Decision and Control (CDC)}, 2019, pp. 7623--7628.

\bibitem{rosolia2020multi}
U.~Rosolia and A.~D. Ames, ``Multi-rate control design leveraging control
  barrier functions and model predictive control policies,'' \emph{IEEE Control
  Systems Letters}, vol.~5, no.~3, pp. 1007--1012, 2020.

\bibitem{CompdelayMPC2004IFAC}
R.~Findeisen and F.~Allg{\"o}wer, ``Computational delay in nonlinear model
  predictive control,'' \emph{IFAC Proceedings Volumes}, vol.~37, no.~1, pp.
  427--432, 2004.

\bibitem{ZanonFastsolver2015}
R.~Quirynen, M.~Vukov, M.~Zanon, and M.~Diehl, ``Autogenerating microsecond
  solvers for nonlinear mpc: A tutorial using acado integrators,''
  \emph{Optimal Control Applications and Methods}, vol.~36, no.~5, pp.
  685--704, 2015.

\bibitem{ChaoFastSolver2017IEEETran}
C.~Shen, B.~Buckham, and Y.~Shi, ``Modified c/gmres algorithm for fast
  nonlinear model predictive tracking control of auvs,'' \emph{IEEE
  Transactions on Control Systems Technology}, vol.~25, no.~5, pp. 1896--1904,
  2017.

\bibitem{OHTSUKA2004Automatica}
T.~Ohtsuka, ``A continuation/gmres method for fast computation of nonlinear
  receding horizon control,'' \emph{Automatica}, vol.~40, no.~4, pp. 563--574,
  2004.

\bibitem{CHARLET1989143}
B.~Charlet, J.~L{\'e}vine, and R.~Marino, ``On dynamic feedback
  linearization,'' \emph{Systems \& Control Letters}, vol.~13, no.~2, pp.
  143--151, 1989.

\bibitem{KONG2023126658}
X.~Kong, M.~A. Abdelbaky, X.~Liu, and K.~Y. Lee, ``Stable feedback
  linearization-based economic mpc scheme for thermal power plant,''
  \emph{Energy}, vol. 268, p. 126658, 2023.

\bibitem{IET2023MPCFL}
M.~Khodaverdian and M.~Malekzadeh, ``Attitude stabilization of spacecraft
  simulator based on modified constrained feedback linearization model
  predictive control,'' \emph{IET Control Theory \& Applications}, vol.~17,
  no.~8, pp. 953--967, 2023.

\bibitem{isidori1985nonlinear}
A.~Isidori, \emph{Nonlinear control systems: an introduction}.\hskip 1em plus
  0.5em minus 0.4em\relax Springer, 1985.

\bibitem{de2002control}
A.~De~Luca, G.~Oriolo, and M.~Vendittelli, ``Control of wheeled mobile robots:
  An experimental overview,'' \emph{RAMSETE: articulated and mobile robotics
  for services and technologies}, pp. 181--226, 2002.

\bibitem{choset2005principles}
H.~Choset, K.~M. Lynch, S.~Hutchinson, G.~A. Kantor, and W.~Burgard,
  \emph{Principles of robot motion: theory, algorithms, and
  implementations}.\hskip 1em plus 0.5em minus 0.4em\relax MIT press, 2005.

\bibitem{hashim2019special}
H.~A. Hashim, ``Special orthogonal group so (3), euler angles, angle-axis,
  rodriguez vector and unit-quaternion: Overview, mapping and challenges,''
  \emph{arXiv preprint arXiv:1909.06669}, 2019.

\bibitem{minh2022safetycritical}
N.~N. Minh, S.~McIlvanna, Y.~Sun, Y.~Jin, and M.~Van, ``Safety-critical model
  predictive control with control barrier function for dynamic obstacle
  avoidance,'' 2022.

\bibitem{slotine1991applied}
J.-J.~E. Slotine, W.~Li \emph{et~al.}, \emph{Applied nonlinear control}.\hskip
  1em plus 0.5em minus 0.4em\relax Prentice hall Englewood Cliffs, NJ, 1991,
  vol. 199, no.~1.

\bibitem{scokaert1996infinite}
P.~O. Scokaert and J.~B. Rawlings, ``Infinite horizon linear quadratic control
  with constraints,'' \emph{IFAC Proceedings Volumes}, vol.~29, no.~1, pp.
  5905--5910, 1996.

\bibitem{richter2011computational}
S.~Richter, C.~N. Jones, and M.~Morari, ``Computational complexity
  certification for real-time mpc with input constraints based on the fast
  gradient method,'' \emph{IEEE Transactions on Automatic Control}, vol.~57,
  no.~6, pp. 1391--1403, 2011.

\bibitem{SHEN2020108863}
C.~Shen and Y.~Shi, ``Distributed implementation of nonlinear model predictive
  control for auv trajectory tracking,'' \emph{Automatica}, vol. 115, p.
  108863, 2020.

\end{thebibliography}
		
	\end{document}